\documentclass[twoside,11pt]{article}

\usepackage[abbrvbib, preprint]{jmlr2e}

\usepackage{lastpage}
\usepackage{hyperref}
\usepackage{amsmath}
\usepackage{amssymb}
\usepackage{mathtools}
\usepackage{graphicx}%
\usepackage{subcaption}  
\usepackage{multirow}%
\usepackage{amsfonts}%
\usepackage{bm}%
\usepackage{mathrsfs}%
\usepackage[title]{appendix}%
\usepackage{xcolor}%
\usepackage{textcomp}%
\usepackage{manyfoot}%
\usepackage{booktabs}%
\usepackage{multirow}
\usepackage{array}  
\usepackage{makecell}
\usepackage{comment}
\usepackage{enumitem}
\usepackage[capitalize,noabbrev]{cleveref}

\newtheorem{assumption}{Assumption}   

\makeatletter
\def\@begintheorem#1#2#3{\trivlist
      \item[\hskip\labelsep{\bfseries #1\ #2.\ #3}]\itshape}
\makeatother

\newcommand{\ba}{\mathbf{a}}
\newcommand{\bb}{\mathbf{b}}
\newcommand{\bc}{\mathbf{c}}

\newcommand{\bbR}{\mathbb{R}}

\newcommand{\bbE}{\mathbb{E}}
\newcommand{\bpsi}{\bm{\psi}}

\newcommand{\cD}{\mathcal{D}}
\newcommand{\cA}{\mathcal{A}}
\newcommand{\cF}{\mathcal{F}}
\newcommand{\cP}{\mathcal{P}}
\newcommand{\cR}{\mathcal{R}}

\newcommand{\cG}{\mathcal{G}}

\newcommand{\cN}{\mathcal{N}}
\newcommand{\cS}{\mathcal{S}}
\newcommand{\cX}{\mathcal{X}}

\newcommand{\btheta}{\boldsymbol{\theta}}

\firstpageno{1}

\begin{document}

\title{Energy Score-Guided Neural Gaussian Mixture Model for \\
  	Predictive Uncertainty Quantification}

\author{\name Yang Yang \email yangyang.stat@gmail.com \\
       \addr School of Statistics and Data Science\\
       Nankai University\\
       Tianjin, China 
       \AND
       \name Chunlin Ji \email chunlin.ji@kuang-chi.com \\
       \addr Kuang-Chi Institute of Advanced Technology\\
       Shenzhen, China
       \AND
       \name Haoyang Li \email haoyang-comp.li@polyu.edu.hk \\
        \addr Department of Computing\\
        The Hong Kong Polytechnic University  \\
        Hong Kong SAR, China 
        \AND
       \name Ke Deng  \email kdeng@tsinghua.edu.cn \\
       \addr Department of Statistics and Data Science\\
       Tsinghua University\\
       Beijing, China}

\editor{}

\maketitle

\begin{abstract}
 Quantifying predictive uncertainty is essential for real-world machine learning applications, especially in scenarios requiring reliable and interpretable predictions. 
 Many common parametric approaches rely on neural networks to estimate distribution parameters by optimizing the negative log-likelihood. 
 However, these methods often encounter challenges like training instability and mode collapse, leading to poor estimates of the mean and variance of the target output distribution. 
 In this work, we propose the Neural Energy Gaussian Mixture Model (NE-GMM), a novel framework that integrates Gaussian Mixture Model (GMM) with Energy Score (ES) to enhance predictive uncertainty quantification. 
 NE-GMM leverages the flexibility of GMM to capture complex multimodal distributions and leverages the robustness of ES to ensure well-calibrated predictions in diverse scenarios. 
 We theoretically prove that the hybrid loss function satisfies the properties of a strictly proper scoring rule, ensuring alignment with the true data distribution, and establish generalization error bounds, demonstrating that the model's empirical performance closely aligns with its expected performance on unseen data.
 Extensive experiments on both synthetic and real-world datasets demonstrate the superiority of NE-GMM in terms of both predictive accuracy and uncertainty quantification.
\end{abstract}

\begin{keywords}
  Uncertainty quantification, Gaussian mixture model, energy score, generalization error, trustable machine learning
\end{keywords}

\section{Introduction}\label{sec:introduction}
Predictive uncertainty quantification (UQ) is a central requirement in modern machine learning systems that operate under distributional shift, limited supervision, or high-stakes decision making, including weather forecasting, healthcare, autonomous driving, and financial markets \citep{smith2014uncertainty,abdar2021review,huang2024amortized,zhang2025uncertainty}. In such applications, a point prediction alone is often insufficient: downstream actions (e.g., allocating medical resources, triggering safety interventions, or managing portfolio risk) typically depend on calibrated measures of uncertainty, such as predictive intervals, tail risks, or probabilities of constraint violations.

Uncertainty in machine learning models can be broadly categorized into two types: epistemic uncertainty, which arises from model limitations and data scarcity, and aleatoric uncertainty, which stems from inherent random noise in the data \citep{amini2020deep, hullermeier2021aleatoric, lahlou2023deup, mucsanyi2024benchmarking, smith2025rethinking}. 
Epistemic uncertainty can be reduced with more data or improved models, while aleatoric uncertainty is irreducible. 
In regression tasks, aleatoric uncertainty can further be classified as homoscedastic (constant variance) or heteroscedastic (input-dependent variance). 
Traditional regression methods, like linear regression, often assume homoscedastic noise, limiting their applicability in real-world settings with varying noise levels.

To address heteroscedastic noise, methods based on the negative log-likelihood (NLL) of a Gaussian distribution are widely used. For example, ensemble of neural networks (NNs) \citep{NIPS2017_9ef2ed4b} have been proposed to improve robustness. However,  NLL optimization often suffers from ``rich-get-richer" effect, where low-variance regions dominate the training process, leading to degraded performance in high-variance regions \citep{skafte2019reliable, pmlr-v119-duan20a,seitzer2022}. To mitigate this issue,  variants such as $\beta$-NLL \citep{seitzer2022} and reliable training schemes \citep{skafte2019reliable} reweight sample contributions during NLL optimization. NGBoost employs a multiparameter boosting algorithm with the natural gradient to efficiently learn the distributions, but exhibits unsmooth oscillations due to the decision-tree base learners, which are inherently piecewise constant across the feature space \citep{pmlr-v119-duan20a}. Furthermore, all these methods rely on a single Gaussian-based likelihood objective, which can be overly sensitive to local variance collapse, and they lack rigorous theoretical analysis. Bayesian approaches, such as Bayesian Neural Networks (BNNs) \citep{kendall2017uncertainties, gawlikowski2023survey} and Monte Carlo dropout \citep{gal2016dropout}, introduce priors over neural network parameters to achieve UQ. While theoretically appealing, these Bayesian methods often come with high computational costs and training instability. Moreover, all the aforementioned methods face challenges in effectively modeling complex non-Gaussian data and multimodal structures.

Mixture Density Networks (MDN) \citep{bishop1994mixture} offer greater flexibility by modeling the output as a Gaussian Mixture Model (GMM), with model parameters learned via neural networks. 
While MDNs can represent multimodality, they are prone to  mode collapse and training instabilities, particularly when optimized purely by NLL, and it still affected by the ``rich-get-richer" effect.  An alternative line of work leverages proper scoring rules defined on samples to train predictive distributions. A prominent example is the Energy Score (ES), which evaluates distributions through expectations of pairwise distances and is applicable in multivariate settings \citep{gneiting2007strictly}. \citet{harakeh2023estimating} proposed SampleNet, a nonparametric approach that treats neural network outputs as empirical distributions and employs ES as the training loss to evaluate predictive distributions. While this method demonstrates strong calibration properties, it suffers from high computational costs of 
$O(M^2)$ due to Monte Carlo sampling, where $M$ is the number of samples, and lacks an explicit parametric structure, limiting both interpretability and computational efficiency.

This paper proposes the \emph{Neural Energy Gaussian Mixture Model (NE-GMM)}, which combines the structural flexibility of Input-dependent Gaussian Mixture Model (IGMM) with the calibration benefits of the Energy Score.
Our key idea is to train an IGMM using a \emph{hybrid loss} that interpolates between NLL and ES.
The NLL component encourages a compact parametric representation with efficient inference, whereas the ES component acts as a stabilizing regularizer that mitigates mode collapse and promotes diversity across mixture components. In summary, our contributions are as follows:
\begin{enumerate}[label=(\arabic*), noitemsep, topsep=0pt, parsep=0pt, partopsep=0pt]
	\item We introduce NE-GMM, an energy-score-guided IGMM trained using a hybrid scoring rule, combining efficient parametric inference with improved calibration.
	\item We derive an analytic expression of the energy score under IGMM (with $K$ mixtures), resulting in an efficient $O(K^2)$ training objective, which is significantly more computationally efficient than SampleNet's $O(M^2)$ complexity when $M\gg K$.
	\item We provide theoretical guarantees: (i) the hybrid score is strictly proper under mild conditions; (ii) the induced learning objective admits generalization error bounds that connect empirical and population risks.
	\item We empirically demonstrate, on synthetic and real-world benchmarks, that NE-GMM improves both predictive accuracy and uncertainty quantification compared to state-of-the-art approaches.
\end{enumerate}

The remainder of the paper is organized as follows. Section~\ref{sec:review} provides background on parametric approaches for UQ, including the single Gaussian distribution, Gaussian mixture model, and energy score. Section~\ref{sec:proposed} introduces the proposed NE-GMM framework, detailing the hybrid loss and statistical inference. The theoretical analysis of the hybrid loss is presented in Section~\ref{sec:theory}. Experimental results on both synthetic and real-world benchmarks are reported in Section~\ref{sec:experiment}, followed by concluding remarks in Section~\ref{sec:discussion}.

\section{Backgrounds and Preliminaries}\label{sec:review}

We consider a regression task with a $d$-dimensional input $x \in \cX\subset\bbR^d$, where $\cX$ is the input space, and the output $y \in \bbR$. Let $\cD = \{(x_i, y_i)\}_{i=1}^N$ represent a set of i.i.d. training samples. The goal is to design a model that accurately predicts $y$ given its corresponding input $x$, while also providing reliable estimates of predictive uncertainty.

\subsection{Single Gaussian Distribution}\label{sec:single_Gaussian}

We assume the conditional distribution of  $y$ given $x$ follows a Gaussian distribution:
\begin{eqnarray}\label{eq:single_Gaussian}
	p(y\mid x) = \phi(y;\mu(x),\sigma^2(x)),
\end{eqnarray}
where $\phi(.;\mu(x),\sigma^2(x))$ is the probability density function of the Gaussian distribution with mean $\mu(x)$ and variance $\sigma^2(x)$. The variance  $\sigma^2(x)$ quantifies the uncertainty of the model.

For an observation $y$ sampled from a distribution $F$ with $f(\cdot)$ as its probability density function, the \emph{logarithmic score} of $F$ with respect to $y$ is defined as
\begin{equation}\label{eq:S_l}
	S_l(F,y)=-\log f(y).
\end{equation}

The Gaussian distribution in \eqref{eq:single_Gaussian} is determined by $\mu(.)$ and $\sigma(.)$, and neural networks are often used to parameterize these functions. Let $\bpsi$ denote the neural network parameters for $\mu(.)$ and $\sigma(.)$, and denote the corresponding Gaussian distribution as $F_{\bpsi}^g(\cdot)$. Assuming $\cD = \{(x_i, y_i)\}_{i=1}^N$ is i.i.d. and generated from \eqref{eq:single_Gaussian}, the negative log-likelihood (NLL) loss for $\cD$ is defined as: 
\begin{eqnarray}\label{eq:Gaussian_loss}
	L_{l,\cD}(F_{\bpsi}^g) = \frac{1}{N}\sum_{i=1}^N S_l(F_{\bpsi}^g(x_i),y_i),
\end{eqnarray}
where 
$$S_l(F_{\bpsi}^g(x),y)=\log(\sigma(x))+\frac{(y-\mu(x))^2}{2\sigma^2(x)}+const.$$

The partial derivatives of $S_l(F_{\bpsi}^g(x),y)$ with respect to $\mu(x)$ and $\sigma(x)$ are 
\begin{eqnarray}\label{eq:Gaussian_Sl_grad}
    \frac{\partial S_l}{\partial \mu(x)} &=& \frac{\mu(x)-y}{\sigma^2(x)},  \\ 
    \frac{\partial S_l}{\partial \sigma(x)} &=& \frac{\sigma^2(x)-(\mu(x)-y)^2}{\sigma^3(x)}.  \nonumber
\end{eqnarray}

Due to the presence of $\sigma(x)$ in the denominator of \eqref{eq:Gaussian_Sl_grad}, the gradients of the NLL loss $L_{l,\cD}$ tend to scale more heavily for low-variance data points, where $\sigma(x)$ is close to 0. This inherently biases the single Gaussian model to focus on low-noise regions, prioritizing accurate predictions in those areas while potentially neglecting high-variance regions. This phenomenon is commonly referred to as the ``rich-get-richer" effect, which can lead to suboptimal uncertainty quantification \citep{skafte2019reliable, pmlr-v119-duan20a, seitzer2022, stirn2023faithful}.

\subsection{Gaussian Mixture Model}
A Gaussian Mixture Model (GMM) is a probability distribution expressed as a weighted sum of $K$ Gaussian components \citep{lindsay1995mixture, everitt2013finite, mclachlan2019finite}. 
The probability density function of an observation $y$ from a GMM is defined as:
\begin{eqnarray}\label{eq:GMM}
	p(y)=  \sum_{k=1}^K \pi_k \phi(y;\mu_k,\sigma_k^2),
\end{eqnarray}
where $\phi(.;\mu_k,\sigma_k^2)$ denotes the probability density function of the $k$-th Gaussian component with mean $\mu_k$ and variance $\sigma_k^2$, $\{\pi_k\}_{k=1}^K$ are weights satisfying $\sum_{k=1}^K\pi_k=1$ and $0<\pi_k<1$ for each $k$, and $\btheta=\{(\pi_k, \mu_k,\sigma_k)\}_{k=1}^K$ is the complete set of GMM parameters. These parameters are typically estimated using the Expectation-Maximization (EM) algorithm \citep{dempster1977maximum, mclachlan1988mixture, mclachlan2019finite}. 

When the output $y$ depends on an input $x$, the GMM in \eqref{eq:GMM} can be extended to an Input-dependent GMM (IGMM) for UQ in regression:
\begin{eqnarray}\label{eq:IGMM}
	p(y\mid x) =  \sum_{k=1}^K \pi_k(x) \phi(y;\mu_k(x),\sigma_k^2(x)),
\end{eqnarray}
where $\sum_{k=1}^K \pi_k(x)=1, 0< \pi_k(x)< 1$ for all $x\in\cX$ and $k$, and the parameters $\btheta$ in GMM are replaced by $x$-dependent parameters $\btheta(x)=\{(\pi_k(x), \mu_k(x),\sigma_k(x))\}_{k=1}^K$, which vary with $x$.
\citet{bishop1994mixture} proposed using neural networks, parameterized by $\bpsi$, to map $x$ to the corresponding IGMM parameters $\btheta(x)$. Here $\btheta(x)$ depends on $\bpsi$ and is denoted as $F_{\bpsi}(x)$, which forms the basis of the so-called Mixture Density Networks (MDN) \citep{bishop1994mixture}.  The function $F_{\bpsi}(.)$ outputs a $3K$-dimensional vector containing all the parameters of IGMM. Since the output $F_{\bpsi}(.)$ defines the IGMM, which models the conditional distribution of $y\mid x$, we also use $F_{\bpsi}(.)$ to refer to the IGMM distribution itself when there is no ambiguity.

Given a set of samples i.i.d. $\cD = \{(x_i, y_i)\}_{i=1}^N$ from an IGMM characterized by parameters $F_{\bpsi}(\cdot)$, the NLL loss for $\cD$ is defined as: 
\begin{eqnarray}\label{eq:IGMM_loss}
	L_{l,\cD}(F_{\bpsi}) = \frac{1}{N}\sum_{i=1}^N S_l(F_{\bpsi}(x_i),y_i),
\end{eqnarray}
where 
$$S_l(F_{\bpsi}(x),y)=-\log \left(\sum_{k=1}^K \pi_k(x)\phi(y;\mu_k(x),\sigma_k^2(x))\right),$$
and $\phi(.;\mu_k(x),\sigma_k^2(x))$ is the probability density function of $\cN(\mu_k(x),\sigma_k^2(x))$. Here $\pi_k(.)$, $\mu_k(.)$ and $\sigma_k(.)$  are determined by the neural network parameters $\bpsi$.

\begin{proposition} \label{prop:IGMM_Sl_grad}
    The partial derivatives of $S_l(F_\psi(x),y)$ with respect to $\pi_k(x)$, $\mu_k(x)$, and $\sigma_k(x)$ in an IGMM are given by:
    
\begin{eqnarray}\label{eq:IGMM_Sl_grad}
	\frac{\partial S_l}{\partial \pi_k(x)} &=& -r_k(x),  \nonumber \\
	\frac{\partial S_l}{\partial \mu_k(x)} &=& \frac{\mu_k(x)-y}{\sigma_k^2(x)} \pi_k(x)r_k(x) ,  \\  
	\frac{\partial S_l}{\partial \sigma_k(x)} &=& \left(\frac{1}{\sigma_k(x)}-\frac{(\mu_k(x)-y)^2}{\sigma_k^3(x)} \right)\pi_k(x)r_k(x),   \nonumber
\end{eqnarray}
where 
\begin{eqnarray*}
	r_k(x)=\frac{\phi(y;\mu_k(x),\sigma_k^2(x))}{\sum_{l=1}^K\pi_{l}(x)\phi(y;\mu_{l}(x),\sigma_{l}^2(x))}.
\end{eqnarray*} 
\end{proposition}

\begin{lemma} \label{lemma:Taylor_expansion_IGMM_Sl_grad}
The Taylor expansions of the partial derivative functions of $S_l(F_{\bpsi}(x),y)$ as $\sigma_k(x)\to \infty$  are:

\begin{eqnarray}\label{eq:Taylor_expansion_Sl_grad}
\frac{\partial S_l}{\partial\pi_k(x)} &=& -\frac{1}{\sqrt{2\pi} T_1} \cdot \frac{1}{\sigma_k(x)} + \frac{\pi_k(x)}{2\pi T_1^2} \cdot \frac{1}{\sigma_k^2(x)}+O\left(\frac{1}{\sigma_k^3(x)}\right), \nonumber \\
\frac{\partial S_l}{\partial\mu_k(x)} &=& 
 \frac{\pi_k(x)(\mu_k(x)-y)}{\sqrt{2\pi}T_1}  \cdot \frac{1}{\sigma_k^3(x)} + O\left(\frac{1}{\sigma_k^4(x)}\right), \\
 \frac{\partial S_l}{\partial\sigma_k(x)} &=& \frac{\pi_k(x)}{\sqrt{2\pi} T_1} \cdot \frac{1}{\sigma_k^2(x)} +O\left(\frac{1}{\sigma_k^3(x)}\right), \nonumber
\end{eqnarray}
where
$$T_1=\sum_{l\neq k} \pi_l(x)\phi(y;\mu_l(x),\sigma_l^2(x)).$$ 
The results holds for $T_1\neq 0$; otherwise, those functions  reduce to the single Gaussian case.
\end{lemma}

From Lemma~\ref{lemma:Taylor_expansion_IGMM_Sl_grad}, the gradients $S_l$ can either vanish or explode as $\sigma_k(\cdot)$ varies. Compared with the single Gaussian case introduced in Section~\ref{sec:single_Gaussian}, the terms $\frac{\partial S_l}{\partial\mu_k(x)}$ and $\frac{\partial S_l}{\partial\sigma_k(x)}$ tend to 0 more quickly as $\sigma_k(x)\to\infty$ in the IGMM setting. This exacerbates the rich-get-richer effect. Furthermore, the vanishing or exploding gradients caused by  $\frac{\partial S_l}{\partial\pi_k(x)}$ will cause mode collapse when multiple mixture components to converge to the same data mode. In this scenario, well-fitted components dominate gradient updates, effectively starving other components, which eventually die out. As a result, the model fails to capture the full multimodality inherent in the overall distribution. Without appropriate regularization, early stopping, or other mitigation techniques, existing methods like MDN are particularly prone to overfitting and training instability, especially when trained on small or noisy datasets.

\subsection{Energy Score}

Energy Score (ES) is a statistical tool designed to evaluate the quality of a predictive distribution $F$ for a target value $y$ by assessing the alignment between $F$ and $y$ as follows: 
\begin{eqnarray}\label{eq:Se}
	S_e(F,y) = \bbE_{z\sim F}\Vert z-y\Vert-\frac{1}{2}\bbE_{z,z'\sim F} \Vert z-z'\Vert,
\end{eqnarray}
where $z$ and $z'$ are independent random variables from $F$, $||x||$ represents the Euclidean norm of the vector $x$, and the expectations are taken with respect to $z$ and $z'$, with $y$ treated as a fixed constant.
The first term in $S_e$ measures the compatibility between $F$ and $y$, while the second term penalizes overconfidence in $F$ by encouraging diversity among its samples. 
By integrating these two terms, the energy score balances prediction accuracy with diversity in the predicted distribution $F$ \citep{gneiting2007strictly}.

If the target value $y$ comes from a distribution $Q$, we can define the expected version of ES with respect to  $Q$ as
\begin{eqnarray}\label{eq:ES_expected}
	S_e(F,Q)=\bbE_{y\sim Q} S_e(F,y)=\bbE_{y\sim Q,z\sim F}\Vert z-y\Vert-\frac{1}{2}\bbE_{z,z'\sim F} \Vert z-z'\Vert,
\end{eqnarray}
where $z,z'\sim F$ and $y\sim Q$ are all independent random variables, and the expectations are taken over all involved random variables.

For a set of target values $\{y_1,\cdots,y_N\}$, we can define the sample averaged version of ES below as a loss function (referred to as the ES loss) to guide the training of the predictive distribution $F_{\bpsi}$ parametrized by the neural network parameters $\bpsi$:
\begin{eqnarray}\label{eq:Le}
	L_{e,\cD}(F_{\bpsi})=\frac{1}{N}\sum_{i=1}^N S_e(F_{\bpsi}(x_i),y_i).
\end{eqnarray}

The ES loss is highly flexible, enabling it to capture diverse predictive uncertainties depending on the choice of the distribution $F_{\bpsi}$. However, its performance heavily relies on the correct selection of $F_{\bpsi}$.  Moreover, when   $F_{\bpsi}$ is complex, the ES loss in \eqref{eq:Le} has no closed form which often depends on using Monte Carlo techniques. For example, \citet{harakeh2023estimating} introduced SampleNet, which optimizes a neural network prediction model $F_{\bpsi}$  with the following approximated ES loss:
$$\hat L_{e,\cD}(F_{\bpsi})=\frac{1}{N}\sum_{i=1}^N \hat{S}_e(F_{\bpsi}^M(x_i),y_i),$$
with $S_e$ replaced by its Monte Carlo version 
$$\hat{S}_e(F_{\bpsi}^M(x_i),y_i)=\frac{1}{M}\sum_{m=1}^M \Vert z_i^{(m)}-y_i\Vert - \frac{1}{2M^2}\sum_{m,n=1}^M \Vert z_i^{(m)}-z_i^{(n)}\Vert.$$
Here, $F_{\bpsi}^M$ denotes the empirical distribution from $M$ samples $\{ z_i^{(1)}, \ldots, z_i^{(M)} \}$ obtained by outputs of the model $F_{\bpsi}$. Because the Monte Carlo approximation $\hat{S}_e$ has a computational complexity of $O(M^2)$, SampleNet is computationally expensive.


\section{Method}\label{sec:proposed}
In this work, we propose the Neural Energy Gaussian Mixture Model (NE-GMM), a novel framework that integrates unique advantages of IGMM and ES for more efficient prediction with UQ. 

\subsection{NE-GMM with a Hybrid Loss}
The conventional MDN models $y$ using an IGMM which often leads to the rich-get-richer effect. The proposed NE-GMM aims to address this limitation by enhancing the MDN framework. It incorporates an additional ES loss as a regularization term alongside the NLL loss.
Unlike the existing method SampleNet, which relies on an approximated ES loss, our approach focuses on utilizing the exact ES loss $L_{e,\cD}(F_{\bpsi})$ in \eqref{eq:Le}. Since $F_{\bpsi}(x)$ is an IGMM, the analytic form of $S_e(F_{\bpsi}(x),y_i)$ can be conveniently derived, as shown in the following theorem.

\begin{theorem}[Analytic form of energy score]\label{thm:int_ES}
	Given the distribution derived from an IGMM $F_{\bpsi}(\cdot)$, the ES loss is
    \begin{eqnarray*}
        S_e(F_{\bpsi}(x),y) = \sum_{m=1}^K \pi_k(x) A_m(x,y)  -\frac{1}{2}\sum_{m=1}^K \sum_{l=1}^K \pi_m(x)\pi_l(x) B_{ml}(x),
    \end{eqnarray*}
    where 
    $$A_m(x,y) = \sigma_m(x) \sqrt{\frac{2}{\pi}} e^{{-\frac{(\mu_m(x)-y)^2}{2\sigma_m^2(x)}}} + (\mu_m(x)-y) \left(2\Phi\left(\frac{\mu_m(x)-y}{\sigma_m(x)}\right)-1\right),$$
    $$B_{ml}(x) =  \sqrt{\sigma_m^2(x)+\sigma_l^2(x)} \sqrt{\frac{2}{\pi}} e^{-\frac{\left(\mu_m(x)-\mu_l(x)\right)^2}{2\left(\sigma_m^2(x)+\sigma_l^2(x)\right)}} + (\mu_m(x)-\mu_l(x)) \left(2\Phi\left(\frac{\mu_m(x)-\mu_l(x)}{\sqrt{\sigma_m^2(x)+\sigma_l^2(x)}}\right)-1\right).$$
\end{theorem}
The calculation of $S_e(F_{\bpsi}(x),y)$ with the computational complexity of $O(K^2)$ is significantly more efficient than the $O(M^2)$ complexity in SampleNet for $M\gg K$. 

Using the expression of $S_e$ derived  in Theorem \ref{thm:int_ES}, we can obtain the following results for the partial derivatives of $S_e$ with respect to $\pi_k(x)$, $\mu_k(x)$ and $\sigma_k(x)$.

\begin{proposition}[Partial derivatives of $S_e$]\label{prop:grad_Se}
	\begin{eqnarray}\label{eq:Se_gradient}
		\frac{\partial S_e}{\partial\pi_k(x)} &=& A_k(x,y)-\sum_{l=1}^K \pi_l(x)B_{kl}(x), \nonumber \\
		\frac{\partial S_e}{\partial\mu_k(x)} &=& \pi_k(x) \left(\left(2\Phi(w_k)-1\right)-\sum_{l=1}^K \pi_l(x)\left(2\Phi(w_{kl})-1\right) \right),  \nonumber \\
		 \frac{\partial S_e}{\partial\sigma_k(x)} &=& 2\pi_k(x) \left( \phi(w_k) - \sum_{l=1}^K \pi_l(x) \frac{\sigma_k(x)}{\sqrt{\sigma_k^2(x)+\sigma_l^2(x)}}\phi(w_{kl})\right), \nonumber
	\end{eqnarray}
	where $w_k=\frac{\mu_k(x)-y}{\sigma_k(x)}, w_{kl}=\frac{\mu_k(x)-\mu_l(x)}{\sqrt{\sigma_k^2(x)+\sigma_l^2(x)}}$.
\end{proposition}

To further understand the behavior of the gradients of  $S_e$, we compute their Taylor expansions.
\begin{lemma} \label{lemma:Taylor_expansion_Se_grad}
As $\sigma_k(x)\to \infty$,

\begin{eqnarray*}
 \frac{\partial S_e}{\partial\pi_k(x)} &=& \frac{(\sqrt{2}-2)\pi_k(x)}{\sqrt{\pi}} \cdot \sigma_k(x)+ \frac{T_2}{\sqrt{2\pi}} \cdot\frac{1}{\sigma_k(x)} +O\left(\frac{1}{\sigma_k^3(x)}\right), \nonumber \\
    \frac{\partial S_e}{\partial\mu_k(x)} &=& \sqrt{\frac{2}{\pi}} T_3 \pi_k(x) \cdot \frac{1}{\sigma_k(x)} + O\left(\frac{1}{\sigma_k^2(x)}\right), \\
    \frac{\partial S_e}{\partial\sigma_k(x)} &=&  \frac{\sqrt{2}-1}{\sqrt{\pi}}\pi_k^2(x) -\frac{T_2 \pi_k(x)}{\sqrt{2\pi}} \cdot \frac{1}{\sigma_k^2(x)} + O\left(\frac{1}{\sigma_k^4(x)}\right), \nonumber
\end{eqnarray*}
where
$$T_2=(\mu_k(x)-y)^2-\sum_{l\neq k} \pi_l(x) (\sigma_l^2(x)+(\mu_k(x)-\mu_l(x))^2),$$ $$T_3=\mu_k(x)-y-\sum_{l\neq k}^K \pi_l(x) (\mu_k(x)-\mu_l(x)).$$
\end{lemma}

Lemma~\ref{lemma:Taylor_expansion_Se_grad} reveals several key properties: $\frac{\partial S_e}{\partial\pi_k(x)}\to\infty$ as $\sigma_k(x)\to\infty$, implying that the ES loss $S_e$ places increased emphasis on high-variance regions when optimizing $\pi_k(x)$. This ensures that components and data points with higher variance are not neglected during optimization. $\frac{\partial S_e}{\partial\sigma_k(x)}\to \frac{\sqrt{2}-1}{\sqrt{\pi}}\pi_k^2(x)>0$ as $\sigma_k(x)\to\infty$. These results indicate that the partial derivatives  with respect to $\pi_k(x)$ and $\sigma_k(x)$ do not exhibit the rich-get-richer effect, which is an issue for the NLL loss $S_l$ in an IGMM. Although $\frac{\partial S_e}{\partial\mu_k(x)}\to 0$  as $\sigma_k(x)\to\infty$, the decay is significantly slower compared to $\frac{\partial S_l}{\partial\mu(x)}$ in a single Gaussian distribution (see \eqref{eq:Gaussian_Sl_grad}) or $\frac{\partial S_l}{\partial\mu_k(x)}$ in an IGMM (see \eqref{eq:IGMM_Sl_grad}). This slower decay ensures that the ES loss remains effective in mitigating the rich-get-richer effect during training. These findings confirm that training neural networks with the ES loss effectively ensures all mixture components contribute meaningfully to the final predictive distribution.

To leverage the complementary strengths of the NLL loss $L_{l,\cD}(F_{\bpsi})$ and the ES loss $L_{e,\cD}(F_{\bpsi})$, we propose a hybrid loss for the predictive model $F_{\bpsi}$:
\begin{eqnarray}\label{eq:CombinedLoss}
	L_{h,\cD}(F_{\bpsi}) = \eta \cdot L_{l,\cD}(F_{\bpsi}) + (1 - \eta) \cdot L_{e,\cD}(F_{\bpsi}) =\frac{1}{N}\sum_{i=1}^N S_h(F_{\bpsi}(x_i),y_i),
\end{eqnarray} 
where $\eta \in [0, 1]$ is a weighting parameter that balances the contributions of the NLL loss and the ES loss, and
\begin{equation}\label{eq:S_h}
    S_h(F,y)=\eta \cdot S_l(F,y) + (1 - \eta) \cdot S_e(F,y)
\end{equation}
is a hybrid score based on the weighted average of $S_l$ and $S_e$.
Hereinafter, we refer to the IGMM equipped with the above energy score as the Neural Energy Gaussian Mixture Model (NE-GMM). The proposed NE-GMM can be trained by minimizing the hybrid loss $L_{h,\cD}(F_{\bpsi})$. 
Apparently, the NE-GMM degenerates to to the MDN framework  with the NLL loss only when $\eta = 1$, and focuses entirely on the ES loss  when $\eta = 0$.

This hybrid framework generalizes and unifies the two approaches, leveraging their respective strengths. 
The NLL loss ensures that the model learns a parametric representation of the predictive distribution, capturing the primary complex and multimodal structure of the data under the framework of MDN. 
Meanwhile, the ES loss addresses the limitations of MDN by promoting diversity in the predicted samples, ensuring that the model mitigates the rich-get-richer effect and avoids mode collapse to capture uncertainty more robustly.
The combined loss function in \eqref{eq:CombinedLoss} allows the model to learn structured and interpretable predictive distributions while encouraging robustness and diversity in its predictions. 
This makes the framework particularly well-suited for applications requiring accurate and reliable representations of complex uncertainties.

\subsection{Prediction with UQ based on NE-GMM}
Once trained, NE-GMM provides a fully specified IGMM for any new input $x$.
This parametric representation enables efficient computation of standard predictive summaries used in regression with UQ.

For an IGMM, the predictive mean and predictive variance can be written in closed form \citep{bishop1994mixture}:
\begin{eqnarray}\label{eq:IGMM_mean_variance}
	m(x) &=& \bbE[y\mid x]=\sum_{k=1}^K \pi_k(x) \mu_k(x), \\ 
	s^2(x) &=& \text{Var}[y\mid x] = \sum_{k=1}^K \pi_k(x) \left( \sigma_k^2(x) + (\mu_k(x) - m(x))^2 \right). \nonumber
\end{eqnarray}
The expression for $s^2(x)$ highlights two complementary sources of uncertainty: the within-component variance $\sigma_k^2(x)$ (local noise) and the between-component dispersion $(\mu_k(x)-m(x))^2$ (multimodality/mixture uncertainty).

In NE-GMM, the functions $(\pi_k(.),\mu_k(.),\sigma_k(.))$ are parameterized by a neural network with parameters $\bpsi$.
After training, we obtain $\hat\bpsi$ by minimizing the hybrid loss $L_{h,\cD}(F_{\bpsi})$.
For any new $x$, we evaluate $F_{\hat\bpsi}(x)=\{(\pi_k(x;\hat\bpsi),\mu_k(x;\hat\bpsi),\sigma_k(x;\hat\bpsi))\}_{k=1}^K$ and compute $m(x)$ and $s^2(x)$ via \eqref{eq:IGMM_mean_variance}.
In practice, $m(x)$ is used as the point prediction, while $s(x)$ provides a scale of predictive uncertainty.

When full predictive intervals are needed, NE-GMM also supports efficient sampling from the mixture: for each $x$,  sample a component index $k\sim\mathrm{Categorical}(\pi_1(x),\ldots,\pi_K(x))$ and then sample $y\sim\cN(\mu_k(x),\sigma_k^2(x))$.
This allows us to approximate predictive quantiles and interval estimates at arbitrary confidence levels without sacrificing the interpretability of the parametric mixture.

\section{Theoretical Results}\label{sec:theory}
This section develops theoretical supports for NE-GMM from two complementary perspectives.
First, we show that the hybrid score $S_{h}$  corresponds to a \emph{strictly proper} scoring rule, implying that, at the population level, minimizing the expected hybrid loss recovers the true conditional distribution.
Second, we provide a finite-sample generalization guarantee that controls the gap between empirical and expected  risks over the function class induced by the neural network parameterization.

Considering probabilistic forecasts on a general sample space $\Omega$. 
Let $\cA$ be a $\sigma$-algebra of subsets of $\Omega$, and $\cP$ be a convex class of probability measures on the measurable space $(\Omega, \cA)$. 

\begin{definition}\label{def:quasi-integrable}
	A real value function $S: \cP\times \Omega\to\bbR$ is ``$\cP$-quasi-integrable" for every $F\in\cP$ if the integral $S(F, Q) = \int S(P, y) dQ(y)$ exists for all $F,Q\in\cP$, though it may not necessarily be finite.
\end{definition}

\begin{definition}\label{def:score_rule}
	A real value function $S: \cP \times \Omega \to \mathbb{R}$ is called a ``scoring rule" if $S(F, \cdot)$ is $\cP$-quasi-integrable for all $F \in \cP$.
\end{definition}

\begin{definition}\label{def:proper_score_rule}
	A scoring rule $S$ is ``proper" relative to $\cP$ if 
	\begin{equation}\label{eq:proper_score_rule}
		S(Q, Q) \leq S(F, Q), \quad \forall F, Q \in \cP,
	\end{equation} 
	and is ``strictly proper" relative to $\cP$ if \eqref{eq:proper_score_rule} holds with equality if and only $F=Q$.
\end{definition}

\begin{theorem}\label{thm:Sh_score_rule}
	Suppose both $S_l$ and $S_e$ are $\cP$-quasi-integrable, the hybrid score $S_h$ in \eqref{eq:S_h} is a strictly proper scoring rule.
\end{theorem}
Theorem \ref{thm:Sh_score_rule} formalizes the key design principle of NE-GMM: combining NLL (logarithmic score) with ES preserves the compatibility property of proper scoring rules, ensuring that the optimization of the hybrid loss effectively approximates the true distribution.

The remaining results focus on finite-sample guarantees for minimizing the empirical hybrid loss.
We state intermediate steps explicitly to clarify the role of each statement: Assumption~\ref{ass:bounded} prevents degenerate parameters and provides tail control for $y$; Lemma~\ref{lemma:y_concentrated} establishes high-probability concentration properties for $y$ that allow us to work on a bounded event; Lemmas~\ref{lemma:Sl_Lipschitz}--\ref{lemma:Se_Lipschitz} show that the per-sample scores are Lipschitz on that event, enabling a contraction argument; Theorem~\ref{thm:GeneralizationError} then bounds the population--empirical hybrid risk gap by the Rademacher complexity of the underlying neural network parameter class.

To further investigate the theoretical properties of NE-GMM in model complexity and training loss of NE-GMM, we introduce the following assumption and definitions.  
\begin{assumption}\label{ass:bounded}
	The mean functions and variance functions in the IGMM are bounded: 
	$|\mu_k(x)|\leq M_{\mu}$ and $\sigma^2_k(x)\in [\sigma^2_{\min},\sigma^2_{\max}]$ for all $k$ and $x\in \cX$, where $M_{\mu}>0, \sigma_{\max}\geq\sigma_{\min}>0$. 
\end{assumption}
The above assumptions for $\mu_k(.)$ and $\sigma_k^2(.)$ can be easily satisfied by adding truncation on the outputs of the neural networks.

\begin{lemma}\label{lemma:y_concentrated}
Under Assumption \ref{ass:bounded}, for any $\alpha\in(0,1)$, there exists $M_c=2M_{\mu}+\sigma_{\max}\sqrt{2\log(2/\alpha)}>0$, such that
$$P\left(| y-\mu_k(x)|\geq M_c \right) \leq \alpha,$$
for all $k$ and $x\in\cX$.
\end{lemma}

\begin{definition}
	Let $\cG$ be a class of real-valued functions mapping from $\cX$ to $\bbR$, and $A=\{a_1,\ldots,a_N\}$ are samples from $\cX$. The empirical Rademacher complexity of $\cG$ is defined as:
	$$\hat{\cR}_A(\cG)=\bbE_{\xi}\left[\sup_{g\in\cG} \left( \frac{1}{N} \sum_{i=1}^N \xi_i g(a_i)\right)\right],$$
	where $\xi_1,\ldots,\xi_N$ are Rademacher variables, which are independent  uniformly chosen from \{-1,1\}.
\end{definition}

\begin{definition}
	The Rademacher complexity of $\cG$ is defined as the expectation of the empirical Rademacher complexity over the samples $A$:
	$$\cR_N(\cG)=\bbE_A\left[\hat{\cR}_A(\cG)\right].$$
\end{definition}

To connect statistical complexity of the neural parameterization to the complexity of the induced loss, we require the scoring functions to be stable under small perturbations of the predicted mixture parameters.
Lemmas~\ref{lemma:Sl_Lipschitz} and \ref{lemma:Se_Lipschitz} provide this stability in the form of high-probability Lipschitz bounds.
These bounds are key for applying a contraction inequality in the Rademacher complexity analysis: they allow us to control the complexity of the composed class $(x,y)\mapsto S_h(F_{\bpsi}(x),y)$ by that of $x\mapsto F_{\bpsi}(x)$.

\begin{lemma}[Lipschitz continuity of $S_l$]\label{lemma:Sl_Lipschitz}
    Under Assumption~\ref{ass:bounded}, for any $\alpha \in (0,1)$, there exists a constant $C_l > 0$ such that the following holds with probability at least $1-\alpha$:
    \begin{eqnarray}
        \Vert S_l(F_{\bpsi}(x),y) - S_l(F_{\bpsi}(x'),y') \Vert_1 \leq C_l \Vert (F_{\bpsi}(x),y) - (F_{\bpsi}(x'),y') \Vert_1. 
    \end{eqnarray}
\end{lemma}

\begin{lemma}[Lipschitz continuity of $S_e$]\label{lemma:Se_Lipschitz}
    Under Assumption~\ref{ass:bounded}, for any $\alpha \in (0,1)$, there exists a constant $C_e > 0$ such that the following holds with probability at least $1-\alpha$:
    \begin{eqnarray}
        \Vert S_e(F_{\bpsi}(x),y) - S_e(F_{\bpsi}(x'),y') \Vert_1 \leq C_e \Vert (F_{\bpsi}(x),y) - (F_{\bpsi}(x'),y') \Vert_1.
    \end{eqnarray}
\end{lemma}

Let $\cF$ be the class of functions that map $\cX$ to $\bbR^{3K}$. We can then apply symmetrization and concentration to obtain a uniform bound on the deviation between expected and empirical hybrid risks over the function class $\cF$. 
\begin{theorem}\label{thm:GeneralizationError}
	For any $\delta \in (0,1)$, with probability at least $1-\delta$, the following inequality holds for all $F_{\bpsi}\in \cF$:
	$$\bbE[L_{h,\cD}(F_{\bpsi})] - L_{h,\cD}(F_{\bpsi}) \leq 4C_h\cR_N(\cF) + M_h\sqrt{\frac{\log(1/\delta)}{2N}},$$
	for some $C_h>0$ and $M_h>0$.
\end{theorem}

Theorem \ref{thm:GeneralizationError} shows that the generalization error is bounded by the neural network's complexity (measured by $\cR_N(\cF)$) and the dataset size $N$. 
$\cR_N (\cF)$ depends on the input dimension and the neural network architecture, including its depth, width, weight norms, and activation function. These aspects have been extensively studied in prior work \cite{neyshabur2015norm,pmlr-v75-golowich18a} and \cite{nips2023_8493e190}, and are not the primary focus of our analysis here.

\section{Experiments}\label{sec:experiment}
To evaluate the performance of the proposed NE-GMM, we conduct experiments on two toy regression tasks with heteroscedastic or multimodal noises following similar settings as in \citet{harakeh2023estimating} and \citet{el2023deep}, and two widely used real-world prediction tasks for UCI regression \citep{hernandez-lobatoc15} and financial time series forecasting \citep{el2023deep}.  
The performance of NE-GMM is compared to five competing methods, including $\beta$-NLL \citep{seitzer2022}, ensemble neural networks (Ensemble-NN) \citep{NIPS2017_9ef2ed4b}, NGBoost \citep{pmlr-v119-duan20a}, MDN \citep{bishop1994mixture}, and SampleNet \citep{harakeh2023estimating}. Hyperparameters of the involved methods are provided in Table~\ref{tab:hyperparameters} of Appendix~\ref{sec:hyperparameter_setting}.

\subsection{Toy Regressions} \label{sec:synthetic}
We begin with two toy regression tasks to evaluate the capability of NE-GMM in modeling heteroscedastic and multimodal noise.
\begin{itemize}
	\item Example 1 - Heteroscedastic Noise: we assume that $y(x)=x\sin(x)+x\varepsilon_1+\varepsilon_2$, where $x\in [-1,11]$, and $\varepsilon_1,\ \varepsilon_2\sim\cN(0, 0.09)$. 
	The true mean and variance functions of this example are $m(x)=x\sin(x)$ and $s^2(x)=0.09(x^2+1)$, respectively.
	\item Example 2 - Bimodal Noise: we assume $y(x)=Ux^3+\varepsilon$, where $U\in\{-1,1\}$ is a random variable following a Bernoulli distribution with $P(U=-1)=0.3$, $x\in [-4,4]$, and $\varepsilon\sim \cN(0,9)$.
	The true conditional distribution of $y\mid x$ under this regression model is an IGMM with two Gaussian components of the following parameters: $\pi=(\pi_1,\pi_2)=(0.3,0.7),\ \mu_1(x)=-x^3,\ \mu_2(x)=x^3,\ \sigma_1(x)=\sigma_2(x)=3$. The overall mean and variance function are $m(x)=0.4x^3$ and $s^2(x)=9+0.138x^6$.
\end{itemize}

For both examples, we generate $N$ training samples ($N=600$ for Example 1 and $N=1000$ for Example 2), along with $0.2N$ validation samples and 300 test samples. Since the true mean function $m(\cdot)$ and standard deviation function $s(\cdot)$ are known, we  evaluate the root mean squared error (RMSE) for $m(\cdot)$ and $s(\cdot)$ on the test sets. The batch size is set to 32, and all methods use a single-layer neural network with 50 hidden units and Tanh activations as feature extractors. The learning rate, chosen from $\{0.001, 0.005, 0.01\}$, is tuned alongside other model parameters based on RMSEs of the combined $m(\cdot)$ and $s(\cdot)$ on the validation set. Each experiment is repeated 50 times, and average results are reported.

For Example 1, both SampleNet and NE-GMM substantially outperform purely likelihood-based approaches ($\beta$-NLL, Ensemble-NN, NGBoost, and MDN) in estimating $m(\cdot)$ and $s(\cdot)$. Importantly, NE-GMM achieves the best overall performance, improving the RMSE of the mean estimator by 5.7\% and the RMSE of the standard deviation estimator by 78.2\% compared to the second-best SampleNet (Table~\ref{tab:rmse_example1}). Figure~\ref{fig:example1} further illustrates that NE-GMM maintains high robustness even in regions with high variance.

In Example 2, the experimental setup is specifically designed to evaluate the ability of methods to represent and learn multimodality. Methods that predict a single parametric distribution or rely solely on sample-based forecasts without explicit mixture components ($\beta$-NLL, Ensemble-NN, NGBoost, and SampleNet) are unable to identify component-specific parameters. These methods fail to recover the bimodal conditional distribution, as they cannot estimate the individual component means $\mu_k(x)$ and standard deviations  $\sigma_k(x)$. Consequently, their performance is assessed only on the overall mean $m(x)$ and standard deviation $s(x)$ as shown in Table~\ref{tab:rmse_example2} and Figure~\ref{fig:example2}. By contrast, NE-GMM accurately estimates the component-specific parameters. It achieves significantly lower RMSE and successfully identifies the individual component means and standard deviations (constant functions uncorrelated with $x$). Furthermore, compared to a standard MDN, NE-GMM is much less prone to mode collapse in this symmetric bimodal setting.

Overall, these toy experiments validate the robustness of NE-GMM in handling both heteroscedastic and multimodal data. Its ability to effectively capture complex uncertainty structures, including mixture representations, demonstrates the practical advantages of incorporating the energy-score regularizer into the training process.

\begin{figure*}[!h]
	\centering
	\includegraphics[width=1\textwidth]{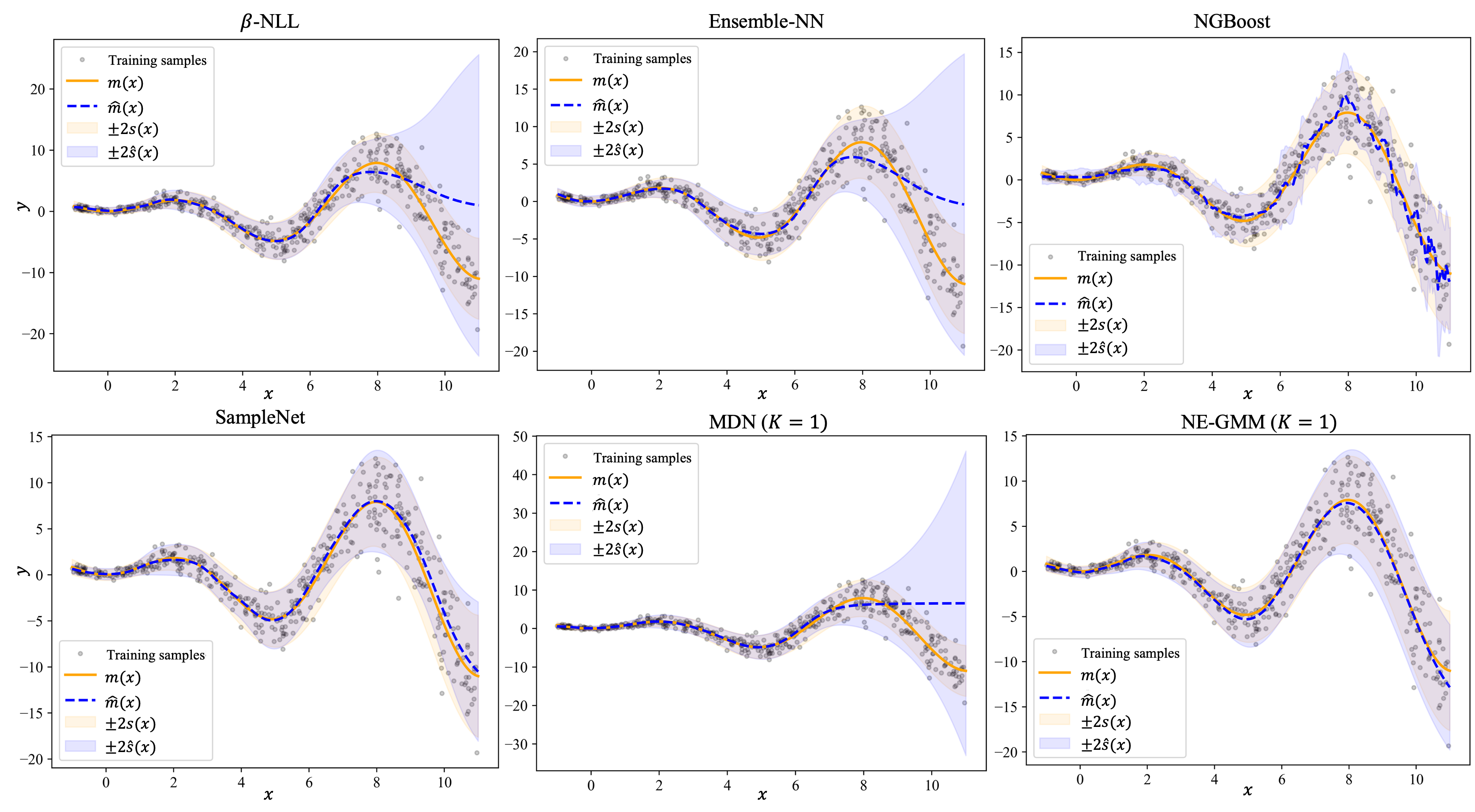}
	\caption{Predictive distribution plots for the toy regression example 1.}
	\label{fig:example1}
\end{figure*}

\begin{figure*}[!h]
	\centering
	\includegraphics[width=1\textwidth]{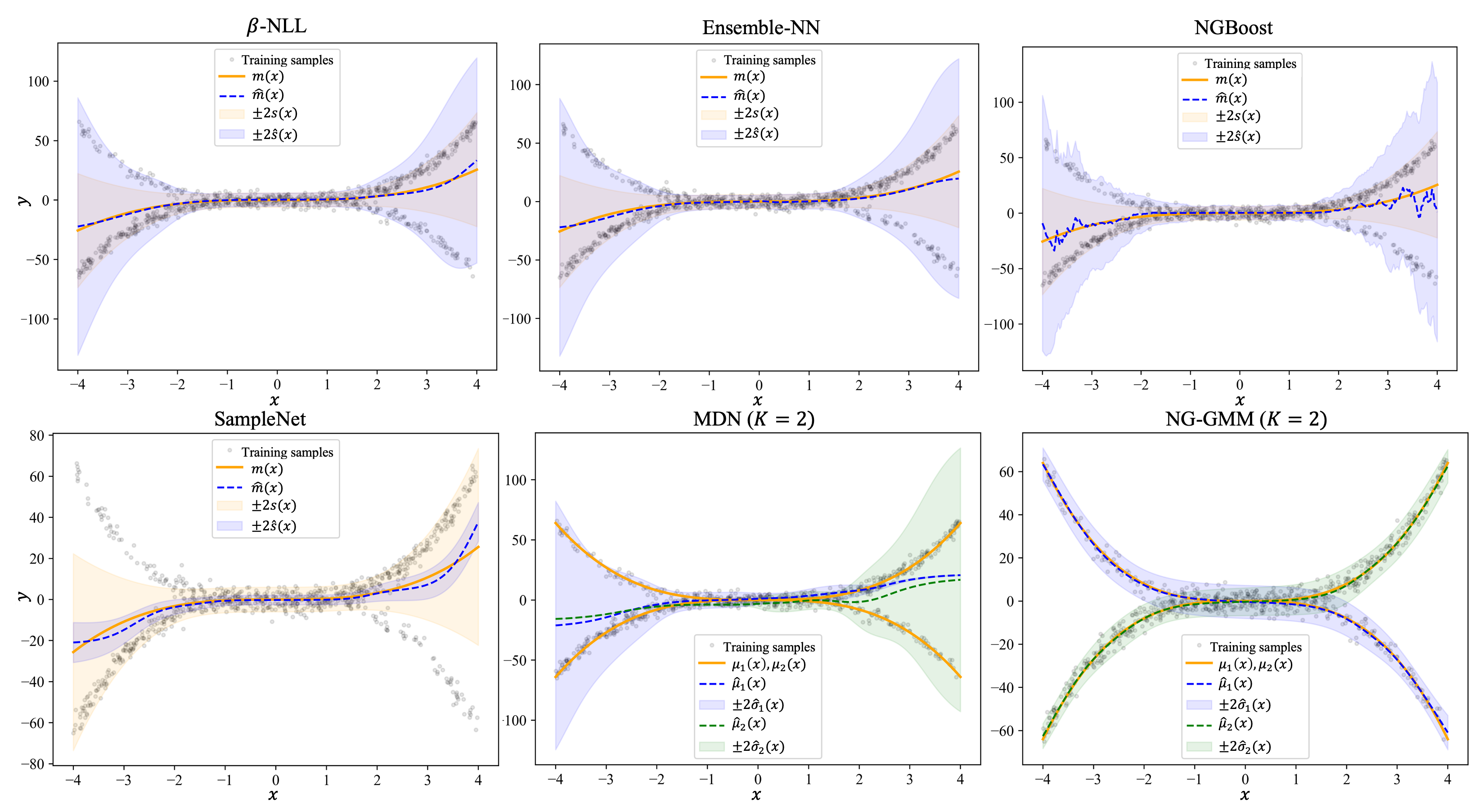}
	\caption{Predictive distribution plots for the toy regression example 2.}
	\label{fig:example2}
\end{figure*}

\begin{table*}[!h]
	\centering
    {\setlength{\tabcolsep}{3pt}
	\caption{The RMSE of $m(\cdot)$ and $s(\cdot)$ for the toy regression  example 1.}
    \label{tab:rmse_example1}
	\begin{tabular}{c|cccccc}
		\hline
		 & $\beta$-NLL & Ensemble-NN & NGBoost & SampleNet & MDN & NE-GMM \\ 
		\hline
		$m(\cdot)$ & 3.669$\pm$1.385 & 3.043$\pm$0.496 & 0.672$\pm$0.063 & 0.454$\pm$0.180 & 4.571$\pm$0.970 & \textbf{0.428$\pm$0.152} \\ 
		$s(\cdot)$ & 2.354$\pm$0.930 & 2.840$\pm$0.465 &	0.434$\pm$0.049 & 0.928$\pm$0.205 & 3.734$\pm$0.945 & \textbf{0.202$\pm$0.064} \\ 
		\hline
	\end{tabular}}
\end{table*}

\begin{table*}[!h]
    {\setlength{\tabcolsep}{3pt}
	\caption{The RMSE of $(m(\cdot), s(\cdot),\pi)$ for the toy regression example 2.}
    \label{tab:rmse_example2}
	\centering
	\begin{tabular}{c|cccccc}
		\hline
		   & $\beta$-NLL & Ensemble-NN & NGBoost & SampleNet & MDN & NE-GMM \\ \hline
		$m(\cdot)$ & 1.903$\pm$0.429 & 2.064$\pm$0.780 &	4.362$\pm$0.862 & 15.054$\pm$0.449 & 10.767$\pm$4.263 & \textbf{1.128$\pm$0.286} \\ 
		$s(\cdot)$ & 12.535$\pm$2.463 & 12.900$\pm$1.021 & 12.4$\pm$0.86 &  7.454$\pm$0.897 &   4.735$\pm$2.644 & \textbf{0.778$\pm$0.260} \\ 
		$\pi$ & None & None & None & None & $\ $0.382$\pm$0.044 & \textbf{0.095$\pm$0.024} \\ \hline
	\end{tabular}}
\end{table*}

\subsection{Regression on UCI Datasets}
To further evaluate the effectiveness of the proposed NE-GMM method, we conduct experiments on UCI benchmark datasets, which are widely used in the literature for real-world regression tasks. Details about the datasets can be found in Appendix \ref{sec:uci_description}. 
Each dataset is randomly split into 20 train-test splits with an 80\%-20\% ratio, except for the large Protein dataset, which is split into 5 train-test splits. 
Inputs and targets are standardized based on the training set, while evaluation metrics such as RMSE and NLL are reported in the data's original scale. 
Additionally, the Prediction Interval Coverage Probability (PICP) and Mean Prediction Interval Width (MPIW) at a 95\% confidence levels are evaluated to assess the reliability and precision of the model’s uncertainty estimates. 
The training set is further split into an 8:2 ratio for training and validation. All methods use the same network architecture: a single-layer neural network with 50 hidden units and ReLU activations. The model was trained for up to 2,000 epochs with a batch size of 32. The learning rate was selected from \{1e-4, 5e-4, 1e-3, 5e-3, 1e-2\} based on the best RMSE on the validation sets. All reported results represent the average across all test splits. It is important to note that these results are not directly comparable to those in other publications, as performance can vary significantly across different data splits. Furthermore, differences in the use of early stopping strategies may also affect comparability.

Tables~\ref{tab:rmse_uci} and \ref{tab:nll_uci} summarize predictive accuracy and distributional fit on the UCI benchmarks.
Overall, NE-GMM achieves the best (or tied-best) RMSE on the majority of datasets, indicating consistently strong point prediction accuracy.
In addition, NE-GMM attains best or near-best NLL on many datasets, suggesting that it not only predicts the mean well but also captures the conditional dispersion effectively.
Even in cases where another method attains the lowest NLL (e.g., Protein), NE-GMM remains competitive, demonstrating robustness across diverse data regimes.

Beyond RMSE and NLL, we assess calibration and sharpness through 95\% prediction intervals.
Appendix Tables~\ref{tab:picp_uci}--\ref{tab:mpiw_uci} show that NE-GMM typically produces PICP closer to the nominal 95\% while maintaining relatively narrow MPIW.
This ``well-calibrated yet sharp'' behavior is particularly desirable for uncertainty quantification and is consistent with the role of the energy-score term as a calibration-oriented regularizer.

Finally, Table~\ref{tab:cost_uci} highlights that NE-GMM is computationally efficient: it is substantially faster than Ensemble-NN and SampleNet, while being comparable to $\beta$-NLL and MDN.
Although NGBoost is often the fastest due to its tree-based structure, it generally lags behind in RMSE and NLL.
Taken together, the UCI results demonstrate that NE-GMM provides a favorable trade-off among accuracy, calibrated uncertainty, and training cost.

\begin{table*}[h]
    {\setlength{\tabcolsep}{1pt}
	\caption{Average RMSE on the test set for UCI regression datasets (``SC" stands for Superconductivity).}
    \label{tab:rmse_uci}
	\centering
	\begin{tabular}{l|cccccc}
		\hline
		Dataset & $\beta$-NLL & Ensemble-NN & NGBoost & SampleNet & MDN & NE-GMM \\
		\hline
		Boston & 2.836$\pm$0.277 &  2.676$\pm$0.178 & 3.130$\pm$0.195 &  2.856$\pm$0.226 & 2.766$\pm$0.195 & \textbf{2.556$\pm$0.205} \\ 
		Concrete & \textbf{5.381$\pm$0.528}  & 5.425$\pm$0.362 & 7.222$\pm$0.020 & 5.935$\pm$0.507 & 6.029$\pm$0.751 & 5.678$\pm$0.725 \\ 
		Energy & 1.944$\pm$0.589  & 2.187$\pm$0.266 & 1.952$\pm$0.216  & 1.866$\pm$0.831 & 1.985$\pm$0.297 & \textbf{1.806$\pm$0.092} \\
		Kin8nm & 0.08$\pm$0.003   &  0.083$\pm$0.003 & 0.197$\pm$0.001 & 0.141$\pm$0.070 & 0.082$\pm$0.001 & \textbf{0.078$\pm$0.001} \\
		Naval & 0.002$\pm$0.000   &  0.004$\pm$0.002 &  0.009$\pm$0.000  & 0.003$\pm$0.001 & \textbf{0.001$\pm$0.000} & \textbf{0.001$\pm$0.000} \\ 
		Protein & 4.072$\pm$0.048   & 4.656$\pm$0.386	& 5.280$\pm$0.033  & 4.370$\pm$0.098 & 4.070$\pm$0.051 & \textbf{4.014$\pm$0.052} \\ 
		SC & 12.588$\pm$0.440   & 13.010$\pm$0.226	& 17.377$\pm$0.288 & 11.943$\pm$0.091 & 11.515$\pm$0.141 & \textbf{11.034$\pm$0.078} \\ 
		WineRed & 0.646$\pm$0.037  &  0.637$\pm$0.027 &	0.644$\pm$0.007  & 0.688$\pm$0.023 & 0.619$\pm$0.014 & \textbf{0.613$\pm$0.017} \\ 
		WineWhite & 0.721$\pm$0.016   & 0.696$\pm$0.012 & 0.747$\pm$0.011  & 0.795$\pm$0.032 & 0.729$\pm$0.021 & \textbf{0.715$\pm$0.023} \\ 
		Yacht & 1.580$\pm$0.196  &  1.346$\pm$0.079 & \textbf{1.197$\pm$0.255} & 4.630$\pm$0.835 & 2.390$\pm$0.131 & 1.284$\pm$0.089 \\ \hline
	\end{tabular}}
\end{table*}

\begin{table*}[h]
	\centering
     {\setlength{\tabcolsep}{1pt}
	\caption{Average NLL on the test set for UCI regression datasets.}
    \label{tab:nll_uci}
	\begin{tabular}{l|cccccc}
		\hline
		Dataset & $\beta$-NLL & Ensemble-NN & NGBoost & SampleNet & MDN & NE-GMM \\
		\hline
		Boston & 4.538$\pm$0.957 &  5.408$\pm$0.590 & 4.259$\pm$0.845  & 4.014$\pm$0.84 & 3.113$\pm$0.397 & \textbf{2.589$\pm$0.125} \\ 
		Concrete & 6.125$\pm$1.087 & 5.437$\pm$1.315	& 6.228$\pm$0.458 & 5.486$\pm$0.958 & 4.501$\pm$0.850 & \textbf{3.649$\pm$0.479} \\
		Energy & 1.645$\pm$0.568 &  \textbf{1.602$\pm$0.415}	& 1.613$\pm$0.040   & 1.924$\pm$0.704 & 1.695$\pm$0.123 & 1.661$\pm$0.206 \\ 
		Kin8nm & \textbf{-1.127$\pm$0.036} &  -0.906$\pm$0.148 &	 -0.936$\pm$0.010  & -0.746$\pm$0.113 & -0.956$\pm$0.033 & -1.026$\pm$0.043 \\ 
		Naval & -6.445$\pm$0.104 &  -5.830$\pm$0.516	& -4.695$\pm$0.478  & -6.399$\pm$0.088 & -6.356$\pm$0.047 & \textbf{-6.730$\pm$0.177} \\ 
		Protein & 3.038$\pm$0.317 &  \textbf{2.528$\pm$0.357} & 2.860$\pm$0.005  & 2.972$\pm$0.124 & 2.776$\pm$0.049 & 2.825$\pm$0.114 \\ 
		SC & 4.185$\pm$0.28 &  4.098$\pm$0.445  & 4.019$\pm$0.023  & 4.548$\pm$0.535 & 4.519$\pm$0.367 & \textbf{3.950$\pm$0.104} \\ 
		WineRed & 3.536$\pm$0.361 &  1.983$\pm$0.433	& 1.520$\pm$0.171  & 3.571$\pm$0.412 & \textbf{1.099$\pm$0.111} & 1.125$\pm$0.151 \\ 
		WineWhite & 2.156$\pm$0.25 &  1.531$\pm$0.099 & 1.459$\pm$0.030  & 1.935$\pm$0.163 & 3.301$\pm$0.261 & \textbf{1.314$\pm$0.167} \\ 
		Yacht & 4.478$\pm$0.522 &  2.462$\pm$0.292  &	2.454$\pm$0.479  & 3.206$\pm$0.624 & 3.808$\pm$0.508 & \textbf{2.247$\pm$0.268} \\ \hline
	\end{tabular}}
\end{table*}

\subsection{Financial time series forecasting}
We conduct one-step-ahead forecasting for financial time series using historical daily price data obtained from Yahoo Finance. The forecasting problem is modeled with a two-layer long short-term memory (LSTM) network \citep{hochreiter1997long}, following the setup in \cite{el2023deep}. The input consists of the closing prices of a specific stock over the previous 30 trading days, while the target is the closing price of the next trading day.

The evaluation is conducted on three datasets, each representing distinct market regimes:
\begin{itemize}
	\item GOOG (stable market regime): Google stock prices (January 2019 – July 2022 for training; August 2022 – January 2023 for testing).
	\item RCL (market shock regime): Royal Caribbean stock prices (January 2019 – April 2020 for training; May 2020 – September 2020 for testing).
	\item GME (high volatility regime): Gamestop stock prices during the ``bubble" period (November 2020 – January 2022 for training; subsequent period for testing)
\end{itemize}

In addition to tuning method-specific parameters, we adjust the LSTM network's learning rate (\{1e-4, 5e-4, 1e-3, 5e-3, 1e-2\}), and dropout rate (\{0.05, 0.1, 0.15, 0.2\}). To perform hyperparameter tuning, we divided the full training period into a split, with 90\% of the data used for training and the remaining 10\% for validation. These were selected based on the configuration that minimized the RMSE on the validation set for each dataset. Each method is trained independently on each dataset for ten runs. We report the mean and standard error of the test RMSE and NLL, highlighting the best-performing method for each experiment in bold. Results are summarized in Tables \ref{tab:rmse_finance} and \ref{tab:nll_finance}.

Overall, NE-GMM consistently achieves superior performance in scenarios with high uncertainty or volatility, such as the RCL (market shock) and GME (high volatility) datasets. Its ability to model complex, multimodal distributions and capture predictive uncertainty enables it to generalize effectively in such challenging conditions. For the GOOG dataset, which reflects a stable market regime with relatively low uncertainty, simpler methods such as SampleNet perform slightly better in terms of RMSE and NLL. This is likely because the  extra mixture flexibility in NE-GMM,  while beneficial in handling uncertain regimes, provides limited advantages under low-uncertainty regimes. Nevertheless, NE-GMM remains competitive and delivers robust results, demonstrating its flexibility across different market conditions.

\begin{table*}[!h]
	\centering
     {\setlength{\tabcolsep}{1pt}
	\caption{Average RMSE on the test set for the financial datasets.}
    \label{tab:rmse_finance}
	\begin{tabular}{l|cccccc}
		\hline
		Dataset & $\beta$-NLL & Ensemble-NN & NGBoost & SampleNet & MDN & NE-GMM \\
		\hline
		GOOG & 6.448$\pm$0.117 &  7.610$\pm$0.126  &  7.790$\pm$0.113 &  \textbf{4.628$\pm$0.487} & 6.782$\pm$0.176 & 6.539$\pm$0.125 \\ 
		RCL & 16.867$\pm$1.684 & 20.648$\pm$0.355  &  19.433$\pm$0.835  & 20.422$\pm$0.729 & 19.160$\pm$1.033 & \textbf{14.578$\pm$0.843} \\ 
		GME & 8.994$\pm$0.403 &  8.799$\pm$0.631  & 7.802$\pm$0.164 & 9.068$\pm$1.349 & 7.190$\pm$1.521 & \textbf{5.737$\pm$0.333} \\ \hline
	\end{tabular}}
\end{table*}

\begin{table*}[!h]
	\centering
    {\setlength{\tabcolsep}{1pt}
	\caption{Average NLL on the test set for the financial datasets.}
    \label{tab:nll_finance}
	\begin{tabular}{l|cccccc}
		\hline
		Dataset & $\beta$-NLL  & Ensemble-NN & NGBoost & SampleNet & MDN & NE-GMM \\
		\hline
		GOOG & 6.784$\pm$0.256 & 7.771$\pm$0.402  &  7.160$\pm$0.618 & \textbf{3.549$\pm$0.901} & 7.244$\pm$1.164 & 4.963$\pm$0.710 \\ 
		RCL & 12.392$\pm$4.210 & 9.672$\pm$1.830  & 10.168$\pm$7.449 & 13.872$\pm$2.474 & 9.071$\pm$4.004 & \textbf{4.638$\pm$0.964} \\ 
		GME & 16.931$\pm$7.318 & 9.472$\pm$1.479  & 8.813$\pm$5.867 & 8.785$\pm$2.704 & 11.475$\pm$2.915 & \textbf{6.354$\pm$1.492} \\ \hline
	\end{tabular}}
\end{table*}

\section{Discussion}\label{sec:discussion}
In this work, we introduced the Neural Energy Gaussian Mixture Model (NE-GMM), a novel framework that enhances the modeling capabilities of IGMM by incorporating the distributional alignment properties of the ES.
By combining these complementary approaches, NE-GMM effectively addresses key challenges in predictive uncertainty estimation for regression tasks, including accurate mean predictions, robust uncertainty quantification, and the ability to handle multimodal and heteroscedastic noise. 
Our theoretical analysis demonstrates that the hybrid loss function employed in NE-GMM corresponds to a strictly proper scoring rule and establishes the generalization error bound.
Empirically, the framework consistently outperforms existing state-of-the-art methods across diverse datasets and scenarios. Overall, NE-GMM provides a simple and efficient way to incorporate a strictly proper, calibration-oriented regularizer into expressive mixture models, yielding improved uncertainty quantification across a broad range of regression settings.

Future works on this research line could focus on the following aspects.
First, adaptive or data-dependent schedules for $\eta$ could be developed to emphasize NLL early for fast convergence and ES later for calibration and robustness.
Second, richer component families (e.g., Student-$t$ components) or structured covariance models could further improve robustness to outliers and heavy tails.

\newpage

\appendix

\section{Analysis of $S_e$ and its Derivatives}\label{app:Se_and_gradient}

\begin{lemma}[Expectation of folded Gaussian distribution]\label{lemma:FoldedGaussian}
    Let $\phi(.)$ and $\Phi(.)$ be the probability density function and cumulative distribution function of $\cN(0,1)$, respectively.
	If $Z\sim \cN(a,b^2)$, we have 
	$$\bbE|Z|=b\sqrt{\frac{2}{\pi}} e^{-\frac{a^2}{2b^2}} + a\left(2\Phi\left(\frac{a}{b}\right)-1\right).$$
\end{lemma}

\begin{proof} 
Based on the definition of $Z$, we have
$$\bbE|Z|=\int_{-\infty}^{\infty} |z| \frac{1}{\sqrt{2\pi}b} e^{-\frac{(z-a)^2}{2b^2}}dz
=\int_{-\infty}^{\infty} |bz+a| \frac{1}{\sqrt{2\pi}} e^{-\frac{z^2}{2}}dz
=\int_{-\frac{a}{b}}^{\infty} (bz+a) \phi(z)dz-\int_{-\infty}^{-\frac{a}{b}} (bz+a) \phi(z)dz.$$
$\bbE|Z|$ is actually the expectation of a folded Gaussian distribution (i.e., the distribution of the absolute value of a Gaussian random variable).
Because
$$\int_{-\frac{a}{b}}^{\infty} (bz+a) \phi(z)dz+\int_{-\infty}^{-\frac{a}{b}} (bz+a) \phi(z)dz=\int_{-\infty}^{\infty} (bz+a) \phi(z)dz=a,$$
we have 
$$\int_{-\frac{a}{b}}^{\infty} (bz+a) \phi(z)dz=a-\int_{-\infty}^{-\frac{a}{b}} (bz+a) \phi(z)dz,$$
and thus
$$\bbE|Z|=a - 2\int_{-\infty}^{-\frac{a}{b}} (bz+a)\phi(z) dz=a-2b\int_{-\infty}^{-\frac{a}{b}}z\phi(z)dz-2a\Phi\left(-\frac{a}{b}\right).$$
Furthermore, 
$$\int_{-\infty}^{-\frac{a}{b}}z\phi(z)dz = \frac{1}{\sqrt{2\pi}}\int_{-\infty}^{-\frac{a}{b}}z e^{-\frac{z^2}{2}}dz = -\frac{1}{\sqrt{2\pi}}\left(e^{-\frac{z^2}{2}}|_{z=-\frac{a}{b}}-e^{-\frac{z^2}{2}}|_{z=-\infty}\right) = -\frac{1}{\sqrt{2\pi}}e^{-\frac{a^2}{2b^2}}. $$

Put these results together, we have 
$$\bbE|Z|=a+2b\frac{1}{\sqrt{2\pi}}e^{-\frac{a^2}{2b^2}}-2a\Phi\left(-\frac{a}{b}\right)=b\sqrt{\frac{2}{\pi}} e^{-\frac{a^2}{2b^2}} + a\left(2\Phi\left(\frac{a}{b}\right)-1\right).$$

\end{proof}

\begin{theorem}[Theorem \ref{thm:int_ES} in the main paper]\label{app:thm:int_ES}
    Given the distribution derived from and IGMM $F_{\bpsi}(\cdot)$, we have
    $$S_e(F_{\bpsi}(x),y)
    =\sum_{m=1}^K \pi_m(x) A_m(x,y)-\frac{1}{2}\sum_{m=1}^K \sum_{l=1}^K \pi_m(x)\pi_l(x) B_{ml}(x),$$ 
    where 
    $$A_m(x,y) = \sigma_m(x) \sqrt{\frac{2}{\pi}} e^{{-\frac{(\mu_m(x)-y)^2}{2\sigma_m^2(x)}}} + (\mu_m(x)-y) \left(2\Phi\left(\frac{\mu_m(x)-y}{\sigma_m(x)}\right)-1\right), $$
    $$B_{ml}(x) =  \sqrt{\sigma_m^2(x)+\sigma_l^2(x)} \sqrt{\frac{2}{\pi}} e^{-\frac{\left(\mu_m(x)-\mu_l(x)\right)^2}{2\left(\sigma_m^2(x)+\sigma_l^2(x)\right)}} + (\mu_m(x)-\mu_l(x)) \left(2\Phi\left(\frac{\mu_m(x)-\mu_l(x)}{\sqrt{\sigma_m^2(x)+\sigma_l^2(x)}}\right)-1\right). $$
\end{theorem}

\begin{proof}
Based on the definition of energy score,  we have:
$$S_e(F_{\bpsi}(x),y)=\bbE\Vert z-y\Vert-\frac{1}{2}\bbE \Vert z-z'\Vert,$$
where $y$ is a fixed real number, $z$ and $z'$ are i.i.d. samples from the IGMM. 
According to the definition of $z$ and $z'$, it is straightforward to see that 
\begin{eqnarray*}
    \bbE||z-y|| &=& \sum_{m=1}^K \pi_m(x)\bbE_{z\sim \cN(\mu_m(x),\sigma_m^2(x))}|z-y|=\sum_{m=1}^K \pi_m(x) A_m(x,y),\nonumber\\
    \bbE||z-z'|| &=& \sum_{m=1}^K\sum_{l=1}^K \pi_m(x)\pi_l(x)\bbE_{z\sim \cN(\mu_m(x),\sigma_m^2(x)),z'\sim \cN(\mu_l(x),\sigma_l^2(x))}|z-z'| \\
    &=& \sum_{m=1}^K\sum_{l=1}^K \pi_m(x)\pi_l(x) B_{ml}(x),
\end{eqnarray*}
where
$$A_m(x,y)=\bbE_{z\sim \cN(\mu_m(x),\sigma^2_m(x))}|z-y|,\quad B_{kl}(x)=\bbE_{z\sim \cN(\mu_m(x),\sigma_m^2(x)),z'\sim \cN(\mu_l(x),\sigma_l^2(x))}|z-z'|.$$
Apparently, $\Vert z-y\Vert=|z-y|$ and $\Vert z-z'\Vert=|z-z'|$ are both mixtures of folded Gaussian distributions in our case. 
Thus, the two expectations in $ S_e(F_{\bpsi}(x),y)$ boil down to obtaining the expectation of folded Gaussian distributions.
Because 
$$z-y\sim\cN(\mu_m(x)-y,\sigma^2_m(x))\quad\mbox{and}\quad z-z'\sim\cN(\mu_m(x)-\mu_l(x),\sigma^2_m(x)+\sigma^2_l(x)),$$
according to Lemma~\ref{lemma:FoldedGaussian}, we get the concrete form of $A_m(x,y)$ and $B_{kl}(x)$ as shown in the theorem.
\end{proof}

The analytic form of $S_e(F_{\bpsi}(x),y)$ can be obtained with the computational complexity of $O(K^2)$. Given the training samples $\cD=\{(x_i,y_i)\}_{i=1}^N$, we have
$$L_{e,\cD}(F_{\bpsi})
=\frac{1}{N}\sum_{i=1}^N S_e(F_{\bpsi}(x_i),y_i)
=\frac{1}{N}\sum_{i=1}^N\left[\sum_{m=1}^K \pi_m(x_i) A_m(x_i,y_i)-\frac{1}{2}\sum_{m=1}^K \sum_{l=1}^K \pi_m(x_i)\pi_l(x_i) B_{ml}(x_i)\right].$$

\begin{proposition}[Proposition~\ref{prop:grad_Se} in the main paper]\label{app:prop:grad_Se}
	\begin{eqnarray}\label{app:eq:Se_gradient}
		\frac{\partial S_e}{\partial\pi_k(x)} &=& A_k(x,y)-\sum_{l=1}^K \pi_l(x)B_{kl}(x), \nonumber \\
		\frac{\partial S_e}{\partial\mu_k(x)} &=& \pi_k(x) \left(\left(2\Phi(w_k)-1\right)-\sum_{l=1}^K \pi_l(x)\left(2\Phi(w_{kl})-1\right) \right),  \nonumber \\
		 \frac{\partial S_e}{\partial\sigma_k(x)} &=& 2\pi_k(x) \left( \phi(w_k) - \sum_{l=1}^K \pi_l(x) \frac{\sigma_k(x)}{\sqrt{\sigma_k^2(x)+\sigma_l^2(x)}}\phi(w_{kl})\right), \nonumber
	\end{eqnarray}
	where $w_k=\frac{\mu_k(x)-y}{\sigma_k(x)}, w_{kl}=\frac{\mu_k(x)-\mu_l(x)}{\sqrt{\sigma_k^2(x)+\sigma_l^2(x)}}$.
\end{proposition}

\begin{proof}
    We decompose $S_e$ into two terms: $$S_e(F_{\bpsi},y)=S_e^{(1)}-\frac{1}{2}S_e^{(2)},$$
    where 
    $$S_e^{(1)} = \sum_{m=1}^K \pi_m(x)A_m(x,y),\quad S_e^{(2)}=\sum_{m=1}^K\sum_{l=1}^K \pi_m(x)\pi_l(x)B_{ml}(x).$$
    For the first term $S_e^{(1)}$, since $A_k(x,y)$ does not depend on $\pi_k(x)$, we have 
    $$\frac{\partial S_e^{(1)}}{\partial \pi_k(x)}=A_k(x,y).$$
    For the second term $S_e^{(2)}$,
    \begin{eqnarray*}
        \frac{\partial (\pi_m(x)\pi_l(x))}{\partial \pi_k(x)} =
    \begin{cases}
        2\pi_k(x), & \text{if } m=k=l, \\
        \pi_l(x),  & \text{if } m=k,l\neq k, \\
        \pi_m(x),  & \text{if } m\neq k,l= k, \\
        0,         & \text{if } m\neq k,l\neq k.
    \end{cases}
    \end{eqnarray*}
    Using the symmetry property $B_{mk}(x)=B_{km}(x)$, we have 
    $$\frac{\partial S_e^{(2)}}{\partial \pi_k(x)} = \sum_{m=1}^K\pi_m(x)B_{mk}(x)+\sum_{l=1}^K\pi_l(x)B_{kl}(x)=2\sum_{m=1}^K\pi_m(x)B_{mk}(x).$$
   Combining these terms, we get:
    $$\frac{\partial S_e}{\partial\pi_k(x)} = A_k(x,y)-\sum_{l=1}^K \pi_l(x)B_{kl}(x).$$

Since $A_m(x,y)$ and $B_{ml}(x)$ are both expectations of the folded Gaussian distribution, we start by considering the function $f(a,b)$  from Lemma~\ref{lemma:FoldedGaussian}:
$$
f(a,b) = \sqrt{\frac{2}{\pi}}b e^{-(a/b)^2/2} + a \left(2\Phi\left(\frac{a}{b}\right) - 1\right), \quad b > 0,\ a \in \mathbb{R}.
$$
The partial derivatives of $f(a,b)$ with respect to $a$ and $b$ are given by:
\begin{eqnarray*}
    \frac{\partial f}{\partial a} = 2\Phi\left(\frac{a}{b}\right)-1,\quad \frac{\partial f}{\partial b} = 2\phi\left(\frac{a}{b}\right).
\end{eqnarray*}
For $A_m(x,y)$ in $S_e^{(1)}$, let $a=\mu_m(x)-y, b=\sigma_m(x)$ and $w_m=a/b$, the partial derivatives with respect to $\mu_k(x)$ and $\sigma_k(x)$ are then:
\begin{eqnarray*}
    \frac{\partial A_m(x,y)}{\partial \mu_k(x)} = 
    \begin{cases}
    2\Phi(w_k)-1, & \text{if } m=k, \\
    0 & \text{if } m\neq k,
    \end{cases}
\end{eqnarray*}
and
\begin{eqnarray*}
    \frac{\partial A_m(x,y)}{\partial \sigma_k(x)} = 
    \begin{cases}
    2\phi(w_k), & \text{if } m=k, \\
    0 & \text{if } m\neq k.
    \end{cases}
\end{eqnarray*}
Therefore:
$$\frac{\partial S_e^{(1)}}{\partial \mu_k(x)} = \pi_k(x)(2\Phi(w_k)-1), \quad \frac{\partial S_e^{(1)}}{\partial \mu_k(x)} = 2\pi_k(x)\phi(w_k). $$

For $B_{ml}(x)$ in $S_e^{(2)}$, let $a=\mu_m(x)-\mu_l(x), b=\sqrt{\sigma_m^2(x)+\sigma_l^2(x)}$, and $w_{ml}=a/b$.  Note the following cases:
when $m=l$,  $B_{ml}(x)=0$;  when $m\neq l$, 
$\frac{\partial a}{\partial \mu_m(x)}=-\frac{\partial a}{\partial \mu_l(x)}=1$. Additionally, by symmetry, $B_{ml}(x)=B_{lm}(x)$. Using these relationships, the partial derivative of $B_{ml}(x,y)$ with respect to $\mu_k(x)$ are:
 
\begin{eqnarray*}
    \frac{\partial B_{ml}(x,y)}{\partial \mu_k(x)} = 
    \begin{cases} 
    0, & \text{if } m=k=l, \\
    2\Phi(w_{kl})-1, & \text{if } m=k,l\neq k, \\
    1-2\Phi(w_{mk}), & \text{if } m\neq k,l= k, \\
    0 & \text{if } m\neq k\neq l.
    \end{cases}
\end{eqnarray*}

Using the fact that $\Phi(w_{kl})=1-\Phi(w_{lk})$, we have 
$2\Phi(w_{kl})-1=1-2\Phi(w_{lk})$, and for $k=l$, $2\Phi(w_{kk})-1=0$. Thus, the partial derivatives of $S_e^{(2)}$ with respect to $\mu_k(x)$ is:

\begin{eqnarray*}
    \frac{\partial S_e^{(2)}}{\partial \mu_k(x)} &=& \pi_k(x)\sum_{m=1}^K \pi_m(x)(1-2\Phi(w_{mk}))+\pi_k(x)\sum_{l=1}^K \pi_l(x)(2\Phi(w_{kl})-1) \\
    &=&2\pi_k(x)\sum_{m=1}^K \pi_m(x)(2\Phi(w_{mk})-1).
\end{eqnarray*}
Combining the results for $S_e^{(1)}$ and $S_e^{(2)}$, we have:
$$\frac{\partial S_e}{\partial\mu_k(x)} = \pi_k(x) \left(\left(2\Phi(w_k)-1\right)-\sum_{l=1}^K \pi_l(x)\left(2\Phi(w_{kl})-1\right) \right).$$

The partial derivative of $S_e$ with respect to $\sigma_k(x)$ can be derived by following similar steps as above using the chain rule $\frac{\partial b}{\partial \sigma_m(x)}=\frac{\sigma_m(x)}{\sqrt{\sigma_m^2(x)+\sigma_l^2(x)}}$.

\end{proof}

\begin{lemma}[Lemmas~\ref{lemma:Taylor_expansion_IGMM_Sl_grad} and \ref{lemma:Taylor_expansion_Se_grad} in the main paper] \label{lemma:Taylor_expansion}
The Taylor expansions of the partial derivative functions of $S_l$ and $S_e$ as $\sigma_k(x)\to \infty$  are given as follows:

\begin{eqnarray*}
\frac{\partial S_l}{\partial\pi_k(x)} &=& -\frac{1}{\sqrt{2\pi} T_1} \cdot \frac{1}{\sigma_k(x)} + \frac{\pi_k(x)}{2\pi T_1^2} \cdot \frac{1}{\sigma_k^2(x)}+O\left(\frac{1}{\sigma_k^3(x)}\right), \\
\frac{\partial S_l}{\partial\mu_k(x)} &=& 
 \frac{\pi_k(x)(\mu_k(x)-y)}{\sqrt{2\pi}T_1}  \cdot \frac{1}{\sigma_k^3(x)} + O\left(\frac{1}{\sigma_k^4(x)}\right), \\
 \frac{\partial S_l}{\partial\sigma_k(x)} &=& \frac{\pi_k(x)}{\sqrt{2\pi} T_1} \cdot \frac{1}{\sigma_k^2(x)} +O\left(\frac{1}{\sigma_k^3(x)}\right), \\
 \frac{\partial S_e}{\partial\pi_k(x)} &=& \frac{(\sqrt{2}-2)\pi_k(x)}{\sqrt{\pi}} \cdot \sigma_k(x)+ \frac{T_2}{\sqrt{2\pi}} \cdot\frac{1}{\sigma_k(x)} +O\left(\frac{1}{\sigma_k^3(x)}\right), \\
    \frac{\partial S_e}{\partial\mu_k(x)} &=& \sqrt{\frac{2}{\pi}} T_3 \pi_k(x) \cdot \frac{1}{\sigma_k(x)} + O\left(\frac{1}{\sigma_k^2(x)}\right), \\
    \frac{\partial S_e}{\partial\sigma_k(x)} &=&  \frac{\sqrt{2}-1}{\sqrt{\pi}}\pi_k^2(x) -\frac{T_2 \pi_k(x)}{\sqrt{2\pi}} \cdot \frac{1}{\sigma_k^2(x)} + O\left(\frac{1}{\sigma_k^4(x)}\right). \\
\end{eqnarray*}
Here $T_1=\sum_{l\neq k} \pi_l(x)\phi(y;\mu_l(x),\sigma_l^2(x))$, $T_2=(\mu_k(x)-y)^2-\sum_{l\neq k} \pi_l(x) (\sigma_l^2(x)+(\mu_k(x)-\mu_l(x))^2)$, and $T_3=\mu_k(x)-y-\sum_{l\neq k}^K \pi_l(x) (\mu_k(x)-\mu_l(x))$. The results holds for $T_1\neq 0$ for the partial derivative functions of $S_l$; otherwise, the partial derivative functions reduce to the single Gaussian case.
\end{lemma}

\begin{proof}
To derive these expansions, we require the Taylor expansions of some basic functions when $z\to 0$:
\begin{eqnarray*}
    e^{z} &=& 1+z+\frac{z^2}{2}+\frac{z^3}{6}+O(z^4), \\
    \sqrt{1+z}&=&1+\frac{z}{2}-\frac{z^2}{8}+\frac{z^3}{16}+O(z^4), \\
    \frac{1}{\sqrt{1+z}} &=& 1-\frac{1}{2}z+\frac{3}{8}z^2+O(z^3),\\
    \frac{1}{1+z}&=&1-z+z^2-z^3+O(z^4),\\
    \Phi(z)&=&\frac{1}{2}+\frac{1}{\sqrt{2\pi}}\left(z-\frac{z^3}{6}\right)+O(z^4).
\end{eqnarray*}

Here we focus on the Taylor expansions of the partial derivative functions of $S_l$ and $S_e$ as $\sigma_k(x)\to\infty$, and the parameters $\{\pi_l(x)\}_{l=1}^K,\{\mu_l(x)\}_{l=1}^K, \{\sigma_l(x)\}_{l\neq k}$ and $y$ are assumed to be fixed.

Since $\phi(y;\mu_l(x),\sigma_l^2(x))>0$, it follows that $T_1=0$ if and only if $\pi_l(x)=0$ for $l\neq k$. This indicates that $\pi_k(x)=1$ and corresponds to the single Gaussian case. Therefore, we focus on the non-trivial case where $T_1 \neq 0$. For the partial derivatives of $S_l$, all terms depend on $r_k(x)$, so it suffices to analyze $r_k(x)$, which is given by:
\begin{eqnarray}\label{eq:rk_decompose}
    r_k(x)=\frac{\phi(y;\mu_k(x);\sigma_k^2(x))}{\pi_k(x)\phi(y;\mu_k(x);\sigma_k^2(x)) + T_1}=\frac{e^{-\frac{(\mu_k(x)-y)^2}{2\sigma_k^2(x)}}/(\sqrt{2\pi}\sigma_k(x))}{T_1\left(\pi_k(x)e^{-\frac{(\mu_k(x)-y)^2}{2\sigma_k^2(x)}}/(\sqrt{2\pi}T_1\sigma_k(x))+1\right)}.
\end{eqnarray}
As $\sigma_k(x)\to\infty$, we have $e^{-\frac{(\mu_k(x)-y)^2}{2\sigma_k^2(x)}}/(\sqrt{2\pi}\sigma_k(x))\to 0$. Using the Taylor expansion of $e^z$ for the numerator and $\frac{1}{1+z}$ for the denominator in \eqref{eq:rk_decompose}, desired result for  $\frac{\partial S_l}{\partial \pi_k(x)}=-r_k(x)$ follows in the lemma.

For the partial derivatives of $S_e$, we compute the Taylor expansion of $\frac{\partial S_e}{\partial \pi_k(x)}$ as an illustration. This requires analyzing the terms in $A_k(x,y)$ and $B_{kl}(x)$. Using the Taylor expansions of  $e^z$ and $\Phi(z)$, we obtain:
$$A_k(x,y)=\sqrt{\frac{2}{\pi}}\sigma_k(x)+\frac{1}{\sqrt{2\pi}}\frac{(\mu_k(x)-y)^2}{\sigma_k(x)}+O\left(\frac{1}{\sigma_k^3(x)}\right).$$
For $B_{kl}(x,y)$, there are two cases: \\
(1) When $l=k$:
$$B_{kl}(x,y)=\frac{2}{\sqrt{\pi}}\sigma_k(x).$$
(2) When $l\neq k$:\\
$$B_{kl}(x,y)=\sqrt{\frac{2}{\pi}}\sigma_k(x)+\frac{(\mu_k(x)-\mu_l(x))^2+\sigma_l^2(x)}{\sqrt{2\pi}}\frac{1}{\sigma_k(x)}.$$
Combining these results and using and the fact that $\pi_k(x)=1-\sum_{l\neq k}\pi_l(x)$, we derive
\begin{eqnarray*}
    \frac{\partial S_e}{\partial \pi_k(x)}&=&A_k(x,y)-\sum_{l=1}^K \pi_l(x) B_{kl}(x) \\
    &=& \frac{(\sqrt{2}-2)\pi_k(x)}{\sqrt{\pi}} \cdot \sigma_k(x)+ \frac{T_2}{\sqrt{2\pi}} \cdot\frac{1}{\sigma_k(x)} +O\left(\frac{1}{\sigma_k^3(x)}\right).
\end{eqnarray*}

The Taylor expansions of $\frac{\partial S_e}{\partial \mu_k(x)}$ and $\frac{\partial S_e}{\partial \sigma_k(x)}$ can be obtained similarly by applying the Taylor expansions of $e^z$, $\frac{1}{\sqrt{1+z}}$ and $\Phi(z)$, and considering the two cases $l=k$ and $l\neq k$.
\end{proof}

\section{Technical Proof of Theorem \ref{thm:Sh_score_rule}}
\label{sec:proof_score_rule}

\subsection{Preparations}
Considering probabilistic forecasts on a general sample space $\Omega$. 
Let $\cA$ be a $\sigma$-algebra of subsets of $\Omega$, and let $\cP$ be a convex class of probability measures on the measurable space $(\Omega, \cA)$. 

\begin{lemma}\label{lemma:nll}
	Suppose the logarithmic score $S_l$ defined in \eqref{eq:S_l} is $\cP$-quasi-integrable, then $S_l$ is strictly proper.
\end{lemma}

\begin{proof}
Let $Q$ be the true distribution of $y$ and $F$ the predictive distribution, with $q$ and $f$ denoting their respective probability density functions. The expectation of $S_l(F,y)$ when $y \sim Q$ is given as follows:
$$S_l(F,Q)=\bbE_{y \sim Q} \big[S_l(F, y)\big] = - \int q(y)\log f(y)dy.$$ 
According the Jensen's inequality, we have 
\begin{eqnarray*}
    S_l(F,Q)-S_l(Q,Q)= -\int q(y) \log \frac{f(y)}{q(y)}dy 
    \geq - \log\int q(y) \frac{f(y)}{q(y)} dy = 0,
\end{eqnarray*} 
and the equation holds if and only if $F=Q$.
\end{proof}

\begin{lemma}\label{lemma:es}
	Suppose the energy score $S_e$ defied in \eqref{eq:Se}  is $\cP$-quasi-integrable, then $S_e$ is strictly proper.
\end{lemma}

\begin{proof}
	The expected energy score under the true distribution $Q$ is
	$S_e(F,Q) = \bbE\Vert y-z\Vert-\frac{1}{2}\bbE \Vert z-z'\Vert$. Then we have
	\begin{eqnarray*}
		S_e(F,Q) -S_e(Q,Q) &=& \int \int \Vert y-z\Vert dF(z) dQ(y)- \frac{1}{2}\int \int \Vert z-z' \Vert dF(z) dF(z') \\ 
		&& -\int \int \Vert y-y'\Vert dQ(y) dQ(y') + \frac{1}{2}\int \int \Vert y-y'\Vert dQ(y) dQ(y')\\  
		&=&  \int \int \Vert y-z\Vert dF(z) dQ(y) - \frac{1}{2}\int \int \Vert z-z' \Vert dF(z) dF(z')  \\
        && -\frac{1}{2}\int \int \Vert y-y'\Vert dQ(y) dQ(y') \\  
		&=&  \bbE_{y\sim Q,z\sim F}\Vert y-z\Vert -\frac{1}{2}\bbE_{z\sim F,z'\sim F}\Vert z-z'\Vert - \frac{1}{2}\bbE_{y\sim Q, y'\sim Q}\Vert y-y'\Vert  \\  
		&=&  2D_{energy}^2(F,Q),
	\end{eqnarray*}
	where $D_{energy}^2(F,Q)$ is the energy distance between $F$ and $Q$, which satisfies $D_{energy}(F,Q)\geq 0$ and it gets equal if and only if $F=Q$.
\end{proof}

\subsection{Proof of Theorem \ref{thm:Sh_score_rule}}
\begin{proof}
	Under the regular conditions where $S_l$ and $S_e$ are $\cP$-quasi-integrable, the hybrid score $S_l$ is also $\cP$-quasi-integrable. 
    Thus, we have
	\begin{eqnarray*}
		S_h(F,Q) -S_h(Q,Q)&=& [\eta S_l(F,Q)+(1-\eta)S_e(F,Q)]-[\eta S_l(Q,Q)+(1-\eta)S_e(Q,Q)] \\
		&=& \eta \underbrace{(S_l(F,Q)-S_l(Q,Q))}_{\geq 0} +(1-\eta)\underbrace{(S_e(F,Q)-S_e(Q,Q))}_{\geq 0} \\
		&\geq& 0,
	\end{eqnarray*}
	where the equation holds if and only if $F=Q$ because $S_l$ and $S_e$ are both strictly proper scoring rules.
	These facts indicate that the hybrid score $S_h$ is also a strictly proper scoring rule.
\end{proof}

\section{Technical Proof of Theorem~\ref{thm:GeneralizationError}}\label{sec:generalization_error}
 
Let $F_{\bpsi}(.)=\{\pi_k(.;\bpsi), \mu_k(.;\bpsi),\sigma_k(.;\bpsi)\}_{k=1}^K$ be the $3K$-dimensional outputs of the NE-GMM model with $\bpsi$ being the parameters of the neural network. Given training samples $\cD=\{(x_i,y_i)\}_{i=1}^N$, the empirical risk (training loss) $L_{h,\cD}(F_{\bpsi})$ and expected risk $\bbE[L_{h,\cD}(F_{\bpsi})]$ are given as follows:
\begin{eqnarray}
	L_{h,\cD}(F_{\bpsi}) &=& \eta L_{l,\cD}(F_{\bpsi}) + (1-\eta)L_{e,\cD}(F_{\bpsi}), \\ \nonumber
	\bbE[L_{h,\cD}(F_{\bpsi})] &=& \eta \bbE[L_{l,\cD}(F_{\bpsi})] + (1-\eta)\bbE[L_{e,\cD}(F_{\bpsi})].
\end{eqnarray}

Next, we analyze the generalization error bound of $L_{h,\cD}$, starting with the tail behavior of $y$, where the Chernoff bound \citep{chernoff1952measure} is used to derive probabilistic bound on the deviations of $y$ from $\mu_k(.)$ in Lemma~\ref{app:lemma:y_concentrated}. Subsequently, the Lipschitz continuity of $S_l$ and $S_e$ is established by bounding their partial derivatives in Lemmas~\ref{app:lemma:Sl_Lipschitz}-\ref{app:lemma:Se_Lipschitz}, ensuring small input variations lead to small changes in the loss. Finally, these results are used with McDiarmid's inequality \citep{{bartlett2002rademacher}} and Rademacher complexity to derive generalization error  bounds for $L_{l,\cD}$ and  $L_{e,\cD}$, and ultimately for $L_{h,\cD}$.

\subsection{Tail Behavior of $y$}

\begin{lemma}[Chernoff bound \citep{chernoff1952measure}]\label{lemma:chernoff_bound}
For a random variable $X\sim \cN(a, b^2)$, and any $t>0$, we have
$$P(| X| \geq t) \leq 2e^{-\frac{(a-t)^2}{2b^2}}.$$
\end{lemma}

Let $Z$ be a latent variable that assigns the example pair $(x,y)$ to a particular mixture component, $\pi_k(x)=P(Z=k)$ be the probability that $(x,y)$ belongs to the $k$-th component, and $y\mid Z=k\sim \cN(\mu_k(x),\sigma_k^2(x))$.

\begin{lemma}[Lemma~\ref{lemma:y_concentrated} in the main paper]\label{app:lemma:y_concentrated}
Under Assumption \ref{ass:bounded}, for any $\alpha\in(0,1)$, there exists $M_c=2M_{\mu}+\sigma_{\max}\sqrt{2\log(2/\alpha)}>0$, such that
$$P\left(| y-\mu_k(x)|\geq M_c \right) \leq \alpha,$$
for all $k$ and $x\in\cX$.
\end{lemma}

\begin{proof}
Under Assumption \ref{ass:bounded}, $|\mu_l(x)-\mu_k(x)|\leq 2M_{\mu}$ for any $k,l$ and $x$. From Lemma \ref{lemma:chernoff_bound}, we can derive:
    \begin{eqnarray*}
        P( | y-\mu_k(x) | \geq M_c) &= & \sum_{l=1}^K \pi_l(x) P( | y-\mu_k(x) | \geq M_c\mid Z=l) \\
        &\leq & 2\sum_{l=1}^K \pi_l(x) \exp\left(-\frac{(M_c - \mu_l(x)+\mu_k(x))^2}{2\sigma_l^2(x)} \right) \\
         &\leq & 2  \exp\left(-\frac{(M_c - 2M_{\mu})^2}{2\sigma_{\max}^2} \right). \\
    \end{eqnarray*}
   Let $M_c=2M_{\mu}+\sigma_{\max}\sqrt{2\log(2/\alpha)}$, we have  $P( | y-\mu_k(x) | \geq M_c)\leq \alpha$.
\end{proof}

\subsection{Lipschitz continuity of $S_l$ and $S_e$}

\begin{lemma}[Lipschitz continuity of $S_l$]\label{app:lemma:Sl_Lipschitz}
    Under Assumption~\ref{ass:bounded}, for any $\alpha \in (0,1)$, there exists a constant $C_l > 0$ such that the following inequality holds with probability at least $1-\alpha$:
    \begin{eqnarray}
        \Vert S_l(F_{\bpsi}(x),y) - S_l(F_{\bpsi}(x'),y') \Vert_1 \leq C_l \Vert (F_{\bpsi}(x),y) - (F_{\bpsi}(x'),y') \Vert_1. 
    \end{eqnarray}
\end{lemma}

\begin{proof}
    Let $\ba=(F_{\bpsi}(x),y)$ and $\bb=(F_{\bpsi}(x'),y')$, where $F_{\bpsi}(x)$ contains $3K$ variables. By the multi-dimensional Mean Value Theorem, there exists $\bc=\ba+t(\bb-\ba)$, for some $t\in[0,1]$, such that:
    $$S_l(\bb)-S_l(\ba)=\boldsymbol{\nabla} S_l(\bc) (\bb-\ba).$$
    The gradient norm can be bounded as follows:
    $$\Vert\boldsymbol{\nabla} S_l(\bc)\Vert_1 \leq \sup_{(F_{\bpsi}(x),y)} \left|\frac{\partial S_l}{\partial y} \right| + \sum_{k=1}^K \left( \left|\frac{\partial S_l}{\partial\pi_k(x)} \right| + \left|\frac{\partial S_l}{\partial\mu_k(x)} \right| + \left|\frac{\partial S_l}{\partial\sigma_k(x)} \right|  \right).  $$
    Now, we proceed to bound each term in the summation individually. Let $\pi_{\min}=\min_{k,x} \pi_k(x)$. Since $0<\pi_k(x)< 1$ for $k,x$, it follows that  $0<\pi_{\min}<1$.
    
    For the term involving $\pi_k(x)$, we have:
    $$\sum_{k=1}^K \left|\frac{\partial S_l}{\partial \pi_k(x)}\right|= \sum_{k=1}^K r_k(x)\leq \frac{1}{\pi_{\min}}.$$
    Using $\sum_{k=1}^K \pi_k(x) r_k(x)=1$ and Lemma~\ref{app:lemma:y_concentrated}, with a probability at least $1-\alpha$, we have:
    $$ \sum_{k=1}^K \left|\frac{\partial S_l}{\partial \mu_k(x)}\right| = \sum_{k=1}^K \left|\frac{y-\mu_k(x)}{\sigma_k^2(x)}\right| \pi_k(x) r_k(x) \leq \frac{M_c}{\sigma_{\min}^2}, $$
    $$\sum_{k=1}^K \left|\frac{\partial S_l}{\partial \sigma_k(x)}\right| = \sum_{k=1}^K \left| \frac{1}{\sigma_k(x)} - \frac{(y-\mu_k(x))^2}{\sigma_k^3(x)} \right| \pi_k(x) r_k(x) \leq \frac{1}{\sigma_{\min}} + \frac{M_c^2}{\sigma_{\min}^3},$$
    and 
    $$\left|\frac{\partial S_l}{\partial y}\right| = \left| \sum_{k=1}^K \frac{\partial S_l}{\partial \mu_k(x)}\right| \leq \sum_{k=1}^K \left| \frac{\partial S_l}{\partial \mu_k(x)}\right|\leq \frac{M_c}{\sigma_{\min}^2}.$$ 
    
    Combining all these bounds, we can set $C_l=1/\pi_{\min}+ (M_c^2+2M_c\sigma_{\min}+ \sigma_{\min}^2)/\sigma_{\min}^3$.
\end{proof}

\begin{lemma}[Lipschitz continuity of $S_e$]\label{app:lemma:Se_Lipschitz}
    Under Assumption~\ref{ass:bounded}, for any $\alpha \in (0,1)$, there exists a constant $C_e > 0$ such that the following holds with probability at least $1-\alpha$:
    \begin{eqnarray}
        \Vert S_e(F_{\bpsi}(x),y) - S_e(F_{\bpsi}(x'),y') \Vert_1 \leq C_e \Vert (F_{\bpsi}(x),y) - (F_{\bpsi}(x'),y') \Vert_1.
    \end{eqnarray}
\end{lemma}
\begin{proof}
Similar to the proof in Lemma~\ref{app:lemma:Sl_Lipschitz}, we only need to bound the gradient norm of $S_e$ with respect to $\pi_k(x)$, $\mu_k(x)$, $\sigma_k(x)$ and $y$.
Using \eqref{eq:Se_gradient}, and noting that the terms $A_k(x,y)$ and $B_{kl}(x)$ are expectations of the fold Gaussian distribution, we again consider the function:
$$
f(a,b) = \sqrt{\frac{2}{\pi}}b e^{-(a/b)^2/2} + a \left(2\Phi\left(\frac{a}{b}\right) - 1\right), \quad b > 0,\ a \in \mathbb{R}.
$$
Since
$$
\frac{\partial f}{\partial b} = 2\phi\left(\frac{a}{b}\right), b > 0,
$$ 
$f(a,b)$ is strictly increasing in $b$. Similarly,
$$
\frac{\partial f}{\partial a} = 2\Phi\left(\frac{a}{b}\right)-1,
$$ 
which implies $f(a,b)$ is strictly increasing in $|a|$. 
Suppose $|a|\leq a_{\max}, b_{\min}\leq b\leq b_{\max}$.  Under these constraints, $f(a,b)$ achieves its maximum at $a=a_{\max}, b=b_{\max}$, and its minimum at $a=0, b=b_{\min}$. 

Using the fact that $e^{-(a/b)^2/2}\leq 1$ and $2\Phi\left(\frac{a}{b}\right) - 1\leq 1$, we can bound  $B_{kl}(x)$ as follows:
\begin{eqnarray*}
    \sqrt{\frac{2}{\pi}}\sigma_{\min} \leq B_{kl}(x) \leq \sqrt{\frac{2}{\pi}} \sigma_{\max} e^{-\frac{(2M_{\mu})^2}{2\sigma_{\max}^2}} + 2M_{\mu} \left(2\Phi\left(\frac{2M_{\mu}}{\sigma_{\max}}\right)-1\right) \leq \sqrt{\frac{2}{\pi}} \sigma_{\max}+2M_{\mu},
\end{eqnarray*}
Moreover, based on Lemma~\ref{app:lemma:y_concentrated}, the following holds with probability at least $1-\alpha$:
\begin{eqnarray*}
    \sqrt{\frac{2}{\pi}}\sigma_{\min}\leq A_k(x,y) \leq \sqrt{\frac{2}{\pi}} \sigma_{\max} e^{-\frac{M_c^2}{2\sigma_{\max}^2}} + M_c \left(2\Phi\left(\frac{M_c}{\sigma_{\max}}\right)-1\right) \leq \sqrt{\frac{2}{\pi}} \sigma_{\max}+M_c, 
\end{eqnarray*}
Using these bounds, we can now compute:
\begin{eqnarray*}
    \sum_{k=1}^K \left|\frac{\partial S_e}{\partial\pi_k(x)} \right|  &\leq & K \left( \sqrt{\frac{2}{\pi}} \sigma_{\max}+M_c -\sqrt{\frac{2}{\pi}}\sigma_{\min}\right)=KM_1,
\end{eqnarray*}
where 
$$M_1=\sqrt{\frac{2}{\pi}} \sigma_{\max}+M_c -\sqrt{\frac{2}{\pi}}\sigma_{\min} = 2M_{\mu}+\left(\sqrt{\frac{2}{\pi}} +\sqrt{2\log(2/\alpha)}\right)\sigma_{\max}- \sqrt{\frac{2}{\pi}}\sigma_{\min}>0.$$
Additionally, we have the following bounds for other terms:
\begin{eqnarray*}
    \sum_{k=1}^K\left|\frac{\partial S_e}{\partial\mu_k(x)} \right| &\leq& \sum_{k=1}^K \pi_k(x) \left(1+\sum_{l=1}^K \pi_l(x) \right)=2, \\
    \sum_{k=1}^K\left| \frac{\partial S_e}{\partial\sigma_k(x)}\right| &\leq&  2\sum_{k=1}^K \pi_k(x)\phi(w_k) = \frac{\sqrt{2}}{\sqrt{\pi}\sigma_{\min}},
\end{eqnarray*}
and
\begin{eqnarray*}
    \left| \frac{\partial S_e}{\partial y}\right| = \left| \sum_{k=1}^K \pi_k(x)\frac{\partial A_k(x,y)}{\partial \mu_k(x)}\right|=
    \left| \sum_{k=1}^K \pi_k(x) \left( 2\Phi(w_k)-1\right) \right| \leq 1.
\end{eqnarray*}
Finally, setting $C_e=3+\sqrt{2/\pi}/\sigma_{\min}+KM_1$ completes the proof.
\end{proof}

\subsection{Generalization Error Analysis for $L_{l,\cD}$ and $L_{e,\cD}$}

\begin{lemma}[McDiarmid’s Inequality \citep{bartlett2002rademacher}]\label{lemma:McDiarmid}
    Let $B$ be some set and $f: B^N \to \bbR$ be a function of $N$ variables such that for some $c_i>0$, for all $i\in [N]$ and for $x_1,\ldots,x_N,x_i'\in B$, we have
	$$ \mid f(x_1,\ldots,x_N) - f(x_1,\ldots,x_{i-1},x_i',x_{i+1},\ldots,x_N) \mid \leq c_i.$$
	Let $X_1,\ldots,X_N$ be independent random variables taking values in $B$. Denoting $f(X_1,\ldots,X_N)$ by $f(B)$ for simplicity. Then, for every $t>0$,
	$$P\{f(B) - \bbE [f(B)] \geq t \} \leq e^{-2t^2/\sum_{i=1}^N c_i^2}.$$
\end{lemma}

\begin{lemma}\label{lemma:Lipschitz_complexity}
	\citep{ledoux2013probability} Let $\cF$ be a class of real functions. If $l: \bbR^d\to \bbR$ is a $\gamma$-Lipschitz function, then $\cR_N(l\circ \cF)\leq 2\gamma \cR_N(\cF)$.
\end{lemma}

In our case, $\cF$ is the class of functions that map $\cX\subset\bbR^d$ to the IGMM's model parameters, outputting $\{\pi_k(.),\mu_k(.),\sigma_k(.)\}_{k=1}^K$. 
\begin{theorem}[Generalization error analysis for $L_{l,\cD}$]\label{thm:Ll_bound}  
    Under Assumption~\ref{ass:bounded}, for any $\delta \in (0,1)$, with probability at least $1-\delta$, for all $F_{\Psi}\in \cF$, there exist a constant $M_l>0$ such that
	$$\bbE[L_{l,\cD}(F_{\bpsi})] - L_{l,\cD}(F_{\bpsi}) \leq 4 C_l\cR_N(\cF) + M_l\sqrt{\frac{\log(2/\delta-1)}{2N}}.$$
\end{theorem}

\begin{proof}
	Let $\cD'$ denote the $i$-th sample of $\cD$ replaced by $(x_i', y_i')$. For the IGMM, the probability density function satisfies
    \begin{eqnarray*}
        \log \left(\sum_{k=1}^K \pi_k(x)\phi(y;\mu_k(x),\sigma_k^2(x))\right) \leq \log \left(\sum_{k=1}^K \pi_k(x) \frac{1}{\sqrt{2\pi}\sigma_{\min}}\right) = -\frac{1}{2} \log(2\pi \sigma_{\min}^2).
    \end{eqnarray*}
    At the same time, using the convexity of the concavity property of the logarithmic function,  we have:
    \begin{eqnarray*}
       \log \left(\sum_{k=1}^K \pi_k(x)\phi(y;\mu_k(x),\sigma_k^2(x))\right) &\geq& \sum_{k=1}^K \pi_k(x) \log \phi(y;\mu_k(x),\sigma_k^2(x)) \\
        &=& -\frac{1}{2} \sum_{k=1}^K \pi_k(x) \left( \log(2\pi\sigma_k^2(x)) + \frac{(y-\mu_k(x))^2}{\sigma_k^2(x)} \right) \\
        &\geq & -\frac{1}{2} \sum_{k=1}^K \pi_k(x) \left( \log(2\pi\sigma_{\max}^2) + \frac{M_c^2}{\sigma_{\min}^2} \right) \\
        &=& -\frac{1}{2}\log(2\pi\sigma_{\max}^2) - \frac{M_c^2}{2\sigma_{\min}^2}. 
    \end{eqnarray*}
   The second-to-last inequality holds with probability at least $1-\delta/2$ by setting $\alpha=\delta/2$ in  Lemma \ref{app:lemma:y_concentrated} and $M_c=2M_{\mu}+\sigma_{\max}\sqrt{2\log(4/\delta)}$ .

    Let $\cD'$ denote the $i$-th sample of $\cD$ replaced by $(x_i', y_i')$. Define 
    $$G(\cD)=\sup_{F_{\bpsi}\in\cF} \left\{\bbE[L_{l,\cD}(F_{\bpsi})] - L_{l,\cD}(F_{\bpsi})\right\}.$$
    We have
    \begin{eqnarray*}
		G(\cD)-G(\cD') & \leq  & \sup_{F_{\bpsi}\in\cF} \left(L_{l,\cD’}(F_{\bpsi}) - L_{l,\cD}(F_{\bpsi})\right) \\
		&=& \sup_{F_{\bpsi}\in\cF} \frac{1}{N} \left(S_l(F_{\bpsi}(x_i'),y_i') - S_l(F_{\bpsi}(x_i),y_i)\right) \\
        &\leq& \frac{1}{N} \left( -\frac{1}{2} \log(2\pi \sigma_{\min}^2) + \frac{1}{2}\log(2\pi\sigma_{\max}^2) + \frac{M_c^2}{2\sigma_{\min}^2}  \right) \\
        &=& \frac{M_l}{N},
	\end{eqnarray*}
     where $M_l=\log(\sigma_{\max}/\sigma_{\min}) + M_c^2/(2\sigma_{\min}^2)$, and the last two inequalities hold with probability at least $1-\delta/2$. Similarly, we also have $G(\cD')-G(\cD)\leq M_l/N$ with probability at least $1-\delta/2$. 
     
     Replacing $c_i$ in Lemma~\ref{lemma:McDiarmid} with $M_l/N$ and $e^{-2t^2/(\sum_{i=1}^N c_i)}$ with 
     $\delta/(2-\delta)$, we have that for any $\delta \in (0,1)$, with probability at least $1-\delta=(1-\delta/2)(1-\delta/(2-\delta))$, 
	\begin{eqnarray*}
		G(\cD)-\bbE[G(\cD)] \leq M_l\sqrt{\frac{\log(2/\delta-1)}{2N}}.
	\end{eqnarray*}

    Since $\bbE[L_{l,\cD}(F_{\bpsi})] - L_{l,\cD}(F_{\bpsi})\leq G(\cD)$, we now need to bound $\bbE[G(\cD)]$. Let $\cD''=\{(x_i'', y_i'')\}_{i=1}^N$ be another set of ghost samples. Suppose we have two new sets, $\cS$ and $\cS'$ in which the $i$-th data points in set $\cD$ and $\cD''$ are swapped with the probability of 0.5.  Because all the samples are independent and identically distributed, $L_{l,\cD^{''}}(F_{\bpsi}) - L_{l,\cD}(F_{\bpsi})$ and $L_{l,\cS^{'}}(F_{\bpsi}) - L_{l,\cS}(F_{\bpsi})$  have the same distribution:
	\begin{eqnarray*}
		L_{l,\cD^{''}}(F_{\bpsi}) - L_{l,\cD}(F_{\bpsi}) &=& \frac{1}{N} \sum_{i=1}^N \left(S_l(F_{\bpsi}(x_i''),y_i'') - S_l(F_{\bpsi}(x_i),y_i)\right), \\
		L_{l,\cS^{'}}(F_{\bpsi}) - L_{l,\cS}(F_{\bpsi}) &=& \frac{1}{N} \sum_{i=1}^N \xi_i \left(S_l(F_{\bpsi}(x_i''),y_i'') - S_l(F_{\bpsi}(x_i),y_i)\right),
	\end{eqnarray*}
	where $\xi_i$ is the Rademacher variables. Based on the above results, 
	\begin{eqnarray*}
		\bbE_{\cD}[G(\cD)] &=& \bbE_{\cD}\left[\sup_{F_{\bpsi}} \bbE[L_{l,\cD}(F_{\bpsi})] - L_{l,\cD}(F_{\bpsi}) \right] \\ \nonumber
		&=& \bbE_{\cD}\left[\sup_{F_{\bpsi}} \bbE_{\cD^{''}} \left[ L_{l,\cD^{''}}(F_{\bpsi}) - L_{l,\cD}(F_{\bpsi})\right] \right] \\
		&\leq& \bbE_{\cD,\cD^{''}} \left[\sup_{F_{\bpsi}} L_{l,\cD^{''}}(F_{\bpsi}) - L_{l,\cD}(F_{\bpsi}) \right] \\
		&=& \bbE_{\cD,\cD^{''},\xi} \left[\sup_{F_{\bpsi}} \frac{1}{N} \sum_{i=1}^N \xi_i \left(S_l(F_{\bpsi}(x_i''),y_i'') - S_l(F_{\bpsi}(x_i),y_i)\right) \right] \\
		&\leq& \bbE_{\cD^{''},\xi} \left[\sup_{F_{\bpsi}} \frac{1}{N} \sum_{i=1}^N \xi_i \left(S_l(F_{\bpsi}(x_i''),y_i'')\right) \right] + \bbE_{\cD,\xi} \left[\sup_{F_{\bpsi}} \frac{1}{N} \sum_{i=1}^N -\xi_i \left(S_l(F_{\bpsi}(x_i),y_i)\right) \right] \\
		&=& \bbE_{\cD^{''}} \left[\hat{\cR}_{\cD^{''}} (S_l\circ \cF)\right] + \bbE_{\cD} \left[\hat{\cR}_{\cD} (S_l\circ \cF)\right] \\
		&=& 2\cR_N (S_l\circ \cF).
	\end{eqnarray*}
Using Lemma~\ref{app:lemma:Sl_Lipschitz} and Lemma~\ref{lemma:Lipschitz_complexity}, $\cR_N (S_l\circ \cF) \leq 4C_l \cR_N (\cF)$ with probability at least $1-\delta/2$, where $C_l=1/\pi_{\min}+ (M_c^2+2M_c\sigma_{\min}+ \sigma_{\min}^2)/\sigma_{\min}^3$. Combining the above results completes the proof.
\end{proof}

\begin{theorem}[Generalization error analysis for $L_{e,\cD}$]\label{thm:Le_bound}
	Under Assumption~\ref{ass:bounded}, for any $\delta \in (0,1)$, with probability at least $1-\delta$, for all $F_{\bpsi}\in \cF$, we have
	$$\bbE[L_{e,\cD}(F_{\bpsi})] - L_{e,\cD}(F_{\bpsi}) \leq 4C_e\cR_N(\cF) + M_e\sqrt{\frac{\log(2/\delta-1)}{2N}}.$$
\end{theorem}

\begin{proof}
	Let $\cD'$ denote the $i$-th sample of $\cD$ replaced by $(x_i', y_i')$. Define $$G(\cD)=\sup_{F_{\bpsi}\in\cF} \left\{\bbE[L_{e,\cD}(F_{\bpsi})] - L_{e,\cD}(F_{\bpsi})\right\}.$$

    Using results from the proof of Lemma~\ref{app:lemma:Se_Lipschitz}, we have the following bounds:
    \begin{eqnarray*}
        \sqrt{\frac{2}{\pi}}\sigma_{\min}\leq A_k(x,y)  &\leq & \sqrt{\frac{2}{\pi}} \sigma_{\max}+M_c,  \\
        \sqrt{\frac{2}{\pi}}\sigma_{\min} \leq B_{kl}(x) &\leq & \sqrt{\frac{2}{\pi}} \sigma_{\max}+2M_{\mu}.
    \end{eqnarray*}
    Setting $\alpha=\delta/2$, the bounds for $A_k(x,y)$ holds with probability at least $1-\alpha$. Therefore, we can derive the following bounds for $S_e(F_{\bpsi}(x),y)$:
    \begin{eqnarray*}
        \sqrt{\frac{2}{\pi}}\sigma_{\min}-\frac{1}{\sqrt{2\pi}} \sigma_{\max}-M_{\mu}\leq S_e(F_{\bpsi}(x),y) \leq \sqrt{\frac{2}{\pi}} \sigma_{\max}+M_c-\frac{1}{\sqrt{2\pi}} \sigma_{\min}.
    \end{eqnarray*}
   Thus, with probability at least $1-\delta/2$, the following inequality holds:
    \begin{eqnarray*}
		G(\cD)-G(\cD') & \leq  & \sup_{F_{\bpsi}\in\cF} \left(L_{e,\cD’}(F_{\bpsi}) - L_{e,\cD}(F_{\bpsi})\right) \\
		&=& \sup_{F_{\bpsi}\in\cF} \frac{1}{N} \left(S_e(F_{\bpsi}(x_i'),y_i') - S_e(F_{\bpsi}(x_i),y_i)\right) \\
        &\leq& \frac{1}{N} \left(\sqrt{\frac{2}{\pi}} \sigma_{\max}+M_c-\frac{1}{\sqrt{2\pi}} \sigma_{\min}-\sqrt{\frac{2}{\pi}}\sigma_{\min}+\frac{1}{\sqrt{2\pi}} \sigma_{\max}+M_{\mu}\right) \\
        &\leq& \frac{M_e}{N},
	\end{eqnarray*}
    where $M_e=3(\sigma_{\max} - \sigma_{\min})/\sqrt{2\pi}+M_c+M_{\mu}$. 
    Similarly, we have $G(\cD')-G(\cD)\leq M_e/N$ with probability at least $1-\delta/2$. Replacing $c_i$ in Lemma~\ref{lemma:McDiarmid} with $M_e/N$ and $e^{-2t^2/(\sum_{i=1}^N c_i)}$ with 
     $\delta/(2-\delta)$, we conclude that for any $\delta \in (0,1)$, with probability at least $1-\delta=(1-\delta/2)(1-\delta/(2-\delta))$, 
	\begin{eqnarray*}
		G(\cD)-\bbE_{\cD}[G(\cD)] \leq 4\cR_N(\cF)+M_e\sqrt{\frac{\log(2/\delta-1)}{2N}}.
	\end{eqnarray*}

\end{proof}

\subsection{Proof of Theorem \ref{thm:GeneralizationError}: Generalization error analysis for $L_{h,\cD}$ }
\begin{proof}
    According to Theorem~\ref{thm:Ll_bound} and Theorem~\ref{thm:Le_bound}, for any $\delta\in(0,1)$, with probability at least $1-\delta$, we have:
	\begin{eqnarray*}
		\bbE[L_{h,\cD}(F_{\bpsi})] - L_{h,\cD}(F_{\bpsi}) &=& \eta (\bbE[L_{l,\cD}(F_{\bpsi})] - L_{l,\cD}(F_{\bpsi})) + (1-\eta)(\bbE[L_{e,\cD}(F_{\bpsi})] - L_{e,\cD}(F_{\bpsi})) \\
		&\leq& \eta \left(4C_l\cR_N(\cF)+ M_l\sqrt{\frac{\log(2/\delta-1)}{2N}}\right) \\
        && +(1-\eta)\left(4C_e\cR_N(\cF)+ M_e\sqrt{\frac{\log(2/\delta-1)}{2N}}\right) \\
		&\leq& (4\eta C_l + 4(1-\eta)C_e) \cR_N(\cF) + (\eta M_l + (1-\eta)M_e) \sqrt{\frac{\log(2/\delta-1)}{2N}}.
	\end{eqnarray*}
    Let $C_h=4\eta C_l + 4(1-\eta)C_e$ and $M_h=\eta M_l + (1-\eta)M_e$. This completes the proof.
\end{proof}

\section{Hyperparameter Settings in Experiments}\label{sec:hyperparameter_setting}

The hyperparameter settings for these methods, applied to the two toy regression examples and two real-world tasks, are summarized in Table  \ref{tab:hyperparameters}. These include  the regularization coefficient $\beta$ in $\beta$-NLL, the ensemble size $T$ in Ensemble-NN, the number of predicted outputs $M$ and the regularization strength $\eta_{reg} $ of the Sinkhorn divergence in SampleNet, the weighting parameter $\eta$ in NE-GMM, and the number of mixture components $K$ in MDN and NE-GMM. For Example 1, where the true data follows an IGMM model with $K=1$, we set $K=1$ for both MDN and NE-GMM. Similarly, for Example 2, where the true data follows an IGMM model with $K=2$, $K=2$ is used for both methods.  Since MDN is a special case of NE-GMM when $\eta=1$, we exclude 
$\eta=1$ when tuning the parameter $\eta$ for NE-GMM.

\begin{table*}[h]
	\centering
    {\setlength{\tabcolsep}{2pt}
	\caption{Hyperparameter settings of different methods across different experiments.}
    \label{tab:hyperparameters}
	\begin{tabular}{c|c|c} 
		\hline
		Methods & Example 1 and 2 & UCI and Financial dataset \\ \hline 
		$\beta$-NLL & $\beta\in \{0, 0.25, 0.5, 0.75, 1\}$ & $\beta\in \{0, 0.25, 0.5, 0.75, 1\}$   \\ \hline 
		Ensemble-NN & $T\in \{3, 5, 8, 10\}$ & $T\in \{3, 5, 8\}$  \\ \hline 
		\multirow{2}{*}{SampleNet} & $M\in\{100, 200, 500, 800\}$, & $M\in\{50, 100, 200, 300\}$,  \\ 
		& $\eta_{reg}\in\{0, 0.1, 0.5, 1, 5\}$ & $\eta_{reg}\in\{0, 0.5, 1, 5\}$  \\ \hline 
		MDN & $K=1$ (Example 1), $K=2$ (Example 2) & $K\in\{5,8,10\}$ \\ \hline 
		\multirow{2}{*}{NE-GMM} & $K=1$ (Example 1), $K=2$ (Example 2),  & $K\in\{5,8,10\}$, \\ 
		& $\eta\in\{0,0.2,0.5,0.8\}$ & $\eta\in\{0,0.2,0.5,0.8\}$  \\ \hline
	\end{tabular}}
\end{table*}

\section{Description and Additional Results of the Real Datasets}
In addition to evaluating different methods in terms of RMSE and NLL, we also assess the Prediction Interval Coverage Probability (PICP) and Mean Prediction Interval Width (MPIW), which are defined as follows:
For the $i$-th sample, let the prediction interval be $[L_i,U_i]$, where $L_i$ and $U_i$ are the lower and upper bounds, respectively, and let the true target value be $y_i$. The PICP is defined as:
\begin{eqnarray*}
    \text{PICP}=\frac{1}{N}\sum_{i=1}^N\mathbb{I}(L_i\leq y_i \leq U_i),
\end{eqnarray*}
where $N$ is the total number of samples, and $\mathbb{I}$ is an indicator function. PICP measures the reliability of the model's uncertainty estimates. The MPIW is defined as:
\begin{eqnarray*}
    \text{MPIW}=\frac{1}{N}\sum_{i=1}^N (U_i - L_i).
\end{eqnarray*}
MPIW quantifies the precision of the model's uncertainty estimates. A well-behaved method should achieve a PICP close to the desired confidence level while maintaining a narrower MPIW. 

All neural networks used in these real dataset experiments were trained on a single  Nvidia V100  GPU. 

\subsection{The UCI Dataset} \label{sec:uci_description}

Table \ref{tab:uci_dataset_stats} summarizes the number of data points $N_{data}$, input dimensions $d_{in}$, output dimensions $d_{out}$ for various datasets used in the UCI regression tasks. Tables \ref{tab:picp_uci}-\ref{tab:mpiw_uci} report the average PICP and MPIW in 20 train-test splits, respectively. Table \ref{tab:cost_uci} summarizes the average training time for UCI regression datasets. 
 
\begin{table}[h!]
	\centering
    {\setlength{\tabcolsep}{10pt}
	\caption{Description of the UCI datasets.}
    \label{tab:uci_dataset_stats}
	\begin{tabular}{l|ccc}
		\toprule
		Dataset & $N_{data}$ & $d_{in}$ &  $d_{out}$ \\
		\midrule
		Boston            & 506      & 13  & 1 \\
		Concrete          & 1030     & 8   & 1 \\
		Energy            & 768      & 8   & 2 \\
		Kin8nm            & 8192     & 8   & 1 \\
		Naval             & 11933    & 16  & 1 \\
		Protein           & 45730    & 9   & 1 \\
		Superconductivity & 21263    & 81  & 1 \\
		WineRed        & 1599     & 11  & 1 \\
		WineWhite     & 4898     & 11  & 1 \\
		Yacht             & 308      & 6   & 1 \\
		\bottomrule
	\end{tabular}}
\end{table}

\begin{table*}[h]
	\centering
    {\setlength{\tabcolsep}{1pt}
	\caption{PICP for 95\% prediction intervals of the UCI regression datasets (``SC" stands for Superconductivity).}
    \label{tab:picp_uci}
	\begin{tabular}{l|cccccc}
    \hline
    Dataset & $\beta$-NLL & Ensemble-NN & NGBoost & SampleNet & MDN & NE-GMM \\
    \hline
    Boston & 0.681$\pm$0.035 & 0.825$\pm$0.034 & \textbf{0.973$\pm$0.016} & 0.251$\pm$0.027 & 0.895$\pm$0.034 & \textbf{0.921$\pm$0.033} \\
    Concrete & 0.806$\pm$0.033 & 0.880$\pm$0.017 & 0.986$\pm$0.006 & 0.271$\pm$0.016 & 0.926$\pm$0.028 & \textbf{0.939$\pm$0.021} \\
    Energy & 0.883$\pm$0.012 & 0.935$\pm$0.014 & 0.988$\pm$0.004 & 0.906$\pm$0.021 & 0.949$\pm$0.016 & \textbf{0.951$\pm$0.004} \\
    Kin8nm & 0.916$\pm$0.007 & \textbf{0.927$\pm$0.006} & 0.977$\pm$0.003 & 0.313$\pm$0.005 & 0.894$\pm$0.013 & 0.905$\pm$0.004 \\
    Naval & 0.967$\pm$0.011 & 0.996$\pm$0.003 & 0.903$\pm$0.004 & 0.977$\pm$0.010 & 0.954$\pm$0.025 & \textbf{0.951$\pm$0.019} \\
    Protein & \textbf{0.946$\pm$0.010} & 0.963$\pm$0.007 & 0.972$\pm$0.001 & 0.605$\pm$0.006 & 0.944$\pm$0.004 & 0.944$\pm$0.002 \\
    SC & 0.912$\pm$0.012 & 0.917$\pm$0.006 & 0.983$\pm$0.003 & 0.350$\pm$0.015 & 0.937$\pm$0.005 & \textbf{0.947$\pm$0.010} \\
    WineRed & 0.789$\pm$0.020 & 0.845$\pm$0.018 & \textbf{0.958$\pm$0.006} & 0.603$\pm$0.051 & 0.922$\pm$0.014 & 0.936$\pm$0.014 \\
    WineWhite & 0.879$\pm$0.011 & 0.915$\pm$0.015 & 0.966$\pm$0.009 & 0.637$\pm$0.033 & 0.934$\pm$0.011 & \textbf{0.942$\pm$0.009} \\
    Yacht & 0.649$\pm$0.063 & 0.916$\pm$0.026 & 0.997$\pm$0.003 & 0.181$\pm$0.039 & 0.853$\pm$0.073 & \textbf{0.939$\pm$0.020} \\
    \hline
	\end{tabular}}
\end{table*}

\begin{table*}[h]
	\centering
    {\setlength{\tabcolsep}{1pt}
	\caption{MPIW for 95\% prediction intervals of the UCI regression datasets.}
    \label{tab:mpiw_uci}
	\begin{tabular}{l|cccccc}
    \hline
    Dataset & $\beta$-NLL & Ensemble-NN & NGBoost & SampleNet & MDN & NE-GMM \\
    \hline
    Boston & 0.579$\pm$0.094 & 0.803$\pm$0.121 & 1.494$\pm$0.060 & 0.705$\pm$0.089 & 1.061$\pm$0.091 & 1.045$\pm$0.189 \\
    Concrete & 0.737$\pm$0.100 & 0.890$\pm$0.078 & 2.086$\pm$0.112 & 0.818$\pm$0.073 & 1.119$\pm$0.054 & 1.123$\pm$0.059 \\
    Energy & 0.407$\pm$0.151 & 0.544$\pm$0.054 & 0.932$\pm$0.009 & 0.578$\pm$0.070 & 0.804$\pm$0.052 & 0.745$\pm$0.063 \\
    Kin8nm & 1.073$\pm$0.054 & 1.164$\pm$0.024 & 3.301$\pm$0.046 & 1.020$\pm$0.017 & 1.001$\pm$0.027 & 0.979$\pm$0.023 \\
    Naval & 0.154$\pm$0.054 & 0.465$\pm$0.124 & 3.010$\pm$0.022 & 0.445$\pm$0.156 & 0.308$\pm$0.096 & 0.328$\pm$0.074 \\
    Protein & 2.809$\pm$0.365 & 3.176$\pm$0.230 & 3.589$\pm$0.029 & 2.074$\pm$0.023 & 6.440$\pm$4.231 & 2.339$\pm$0.045 \\
    SC & 1.152$\pm$0.084 & 1.212$\pm$0.077 & 2.199$\pm$0.035 & 1.061$\pm$0.082 & 0.946$\pm$0.037 & 1.333$\pm$0.976 \\
    WineRed & 2.141$\pm$0.122 & 2.372$\pm$0.127 & 3.312$\pm$0.080 & 2.077$\pm$0.237 & 2.832$\pm$0.083 & 2.894$\pm$0.121 \\
    WineWhite & 2.437$\pm$0.015 & 2.589$\pm$0.026 & 3.615$\pm$0.077 & 2.350$\pm$0.153 & 2.613$\pm$0.061 & 2.790$\pm$0.058 \\
    Yacht & 0.062$\pm$0.020 & 0.145$\pm$0.032 & 0.739$\pm$0.031 & 0.103$\pm$0.038 & 0.532$\pm$0.483 & 0.397$\pm$0.078 \\
        \hline
	\end{tabular}}
\end{table*}

\begin{table*}[h]
	\centering
	\caption{Average training time (in seconds) of the UCI regression datasets.}
	\begin{tabular}{l|cccccc}
		\hline
		Dataset & $\beta$-NLL & Ensemble-NN & NGBoost & SampleNet & MDN & NE-GMM \\
		\hline
		Boston            & 179              & 740                 & 1.479            & 263               & 216          & 227             \\ 
		Concrete          & 227              & 395                 & 1.333            & 420               & 224          & 292             \\ 
		Energy            & 207              & 347                 & 2.777            & 360               & 227          & 238             \\ 
		Kin8nm            & 711              & 1174                & 1.507            & 2068              & 958          & 1171            \\ 
		Naval             & 937              & 1708                & 2.963            & 3086              & 1056         & 1485            \\ 
		Protein           & 3286             & 5906                & 1.828            & 12043             & 4344         & 6622            \\ 
		SC & 1600             & 2889                & 2.920            & 5555              & 1981         & 3098            \\ 
		Winered           & 272              & 409                 & 1.362            & 504               & 257          & 338             \\ 
		Winewhite         & 493              & 833                 & 1.394            & 1488              & 710          & 721             \\ 
		Yacht             & 269              & 482                 & 1.431            & 370               & 266          & 288             \\ \hline
	\end{tabular}
	\label{tab:cost_uci}
\end{table*}

\subsection{The Financial Time Series Forecasting Dataset}\label{sec:financial_description}

Figure \ref{fig:financial} presents the trend plots of three stock datasets. To further quantify the volatility of these three stocks, we use the following metrics:
\begin{itemize}
	\item Daily Return Standard Deviation (Daily Return Std. Dev.): Captures the standard deviation of the stock's daily returns, quantifying day-to-day variability. This is a key metric for assessing short-term volatility.
	\item Annualized Volatility: Scales the daily return standard deviation to an annual level, assuming 252 trading days in a year. It provides a normalized measure of a stock’s overall risk over the course of a year, making it useful for cross-stock comparisons.
	\item Max Drawdown: Represents the maximum observed loss from a peak to a trough in the stock’s price before a new peak is achieved. This is a critical risk metric for evaluating the potential downside of an investment.
\end{itemize}
These metrics of the three stocks are shown in Table \ref{tab:stock_volatility}. The stock price of GOOG remained relatively stable with minimal fluctuations, RCL's stock price was halved during the COVID-19 pandemic, and GME's stock price exhibited significant volatility. Tables \ref{tab:picp_finance}-\ref{tab:mpiw_finance} report the average PICP and MPIW in ten replicates, respectively. Table \ref{tab:cost_finance} summarizes the average training time for the financial datasets. 

\begin{figure}[!htbp]
	\centering
	\begin{subfigure}[h]{0.6\textwidth}
		\centering
		\includegraphics[width=\textwidth]{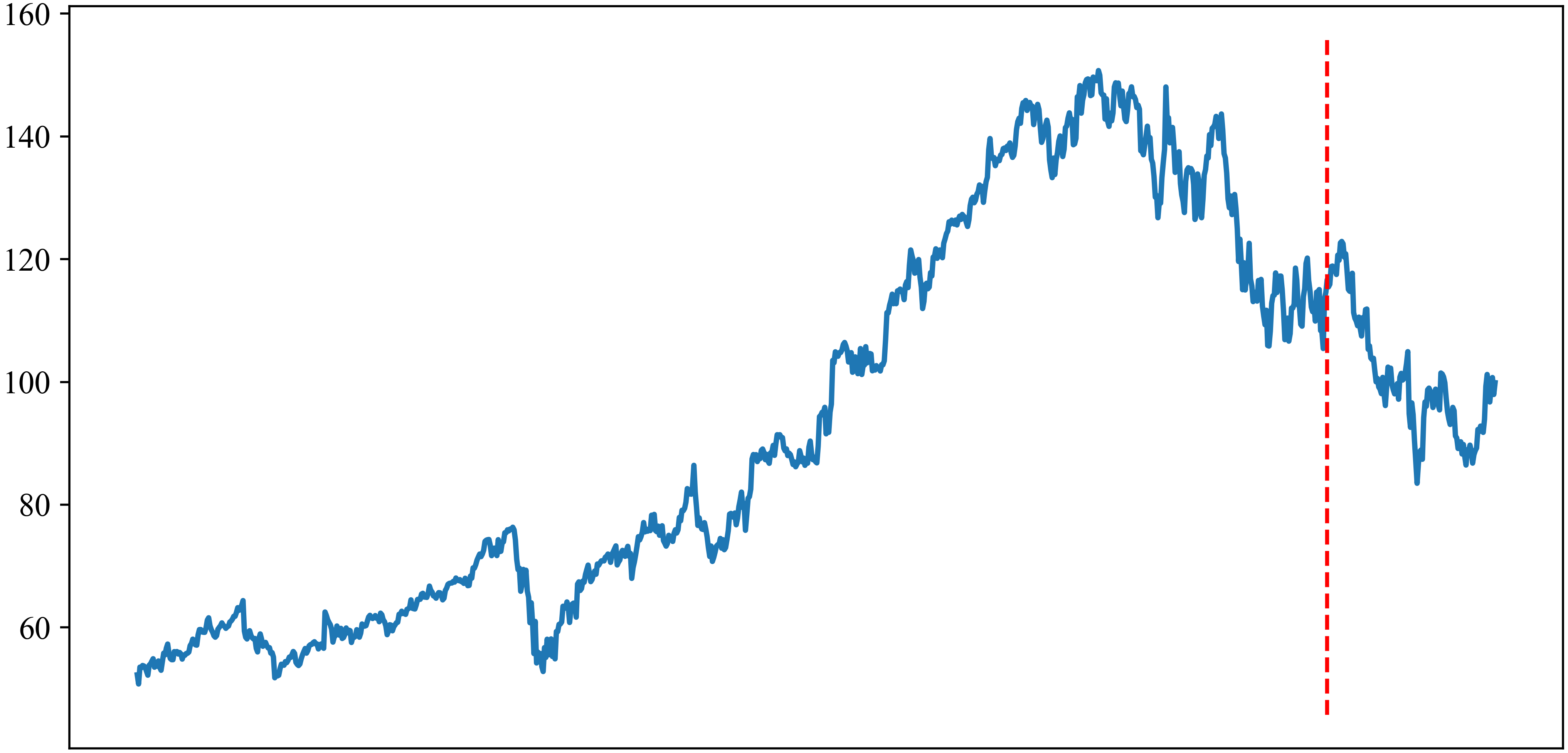}
		\caption{GOOG}
		\label{fig:goog}
	\end{subfigure}
	\vspace{0.5cm} 
	\begin{subfigure}[h]{0.6\textwidth}
		\centering
		\includegraphics[width=\textwidth]{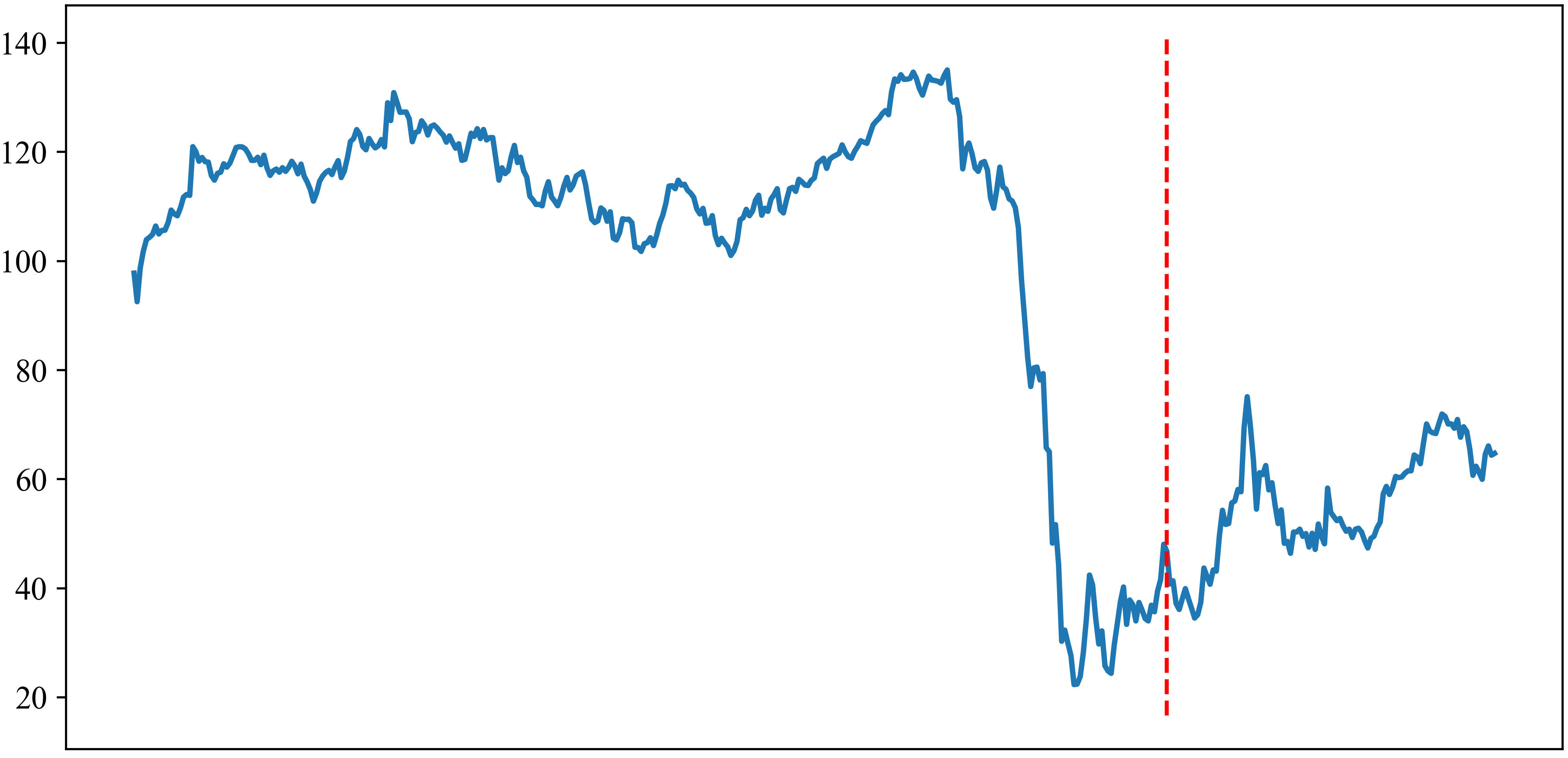}
		\caption{RCL}
		\label{fig:rcl}
	\end{subfigure}
	\vspace{0.5cm}
	\begin{subfigure}[h]{0.6\textwidth}
		\centering
		\includegraphics[width=\textwidth]{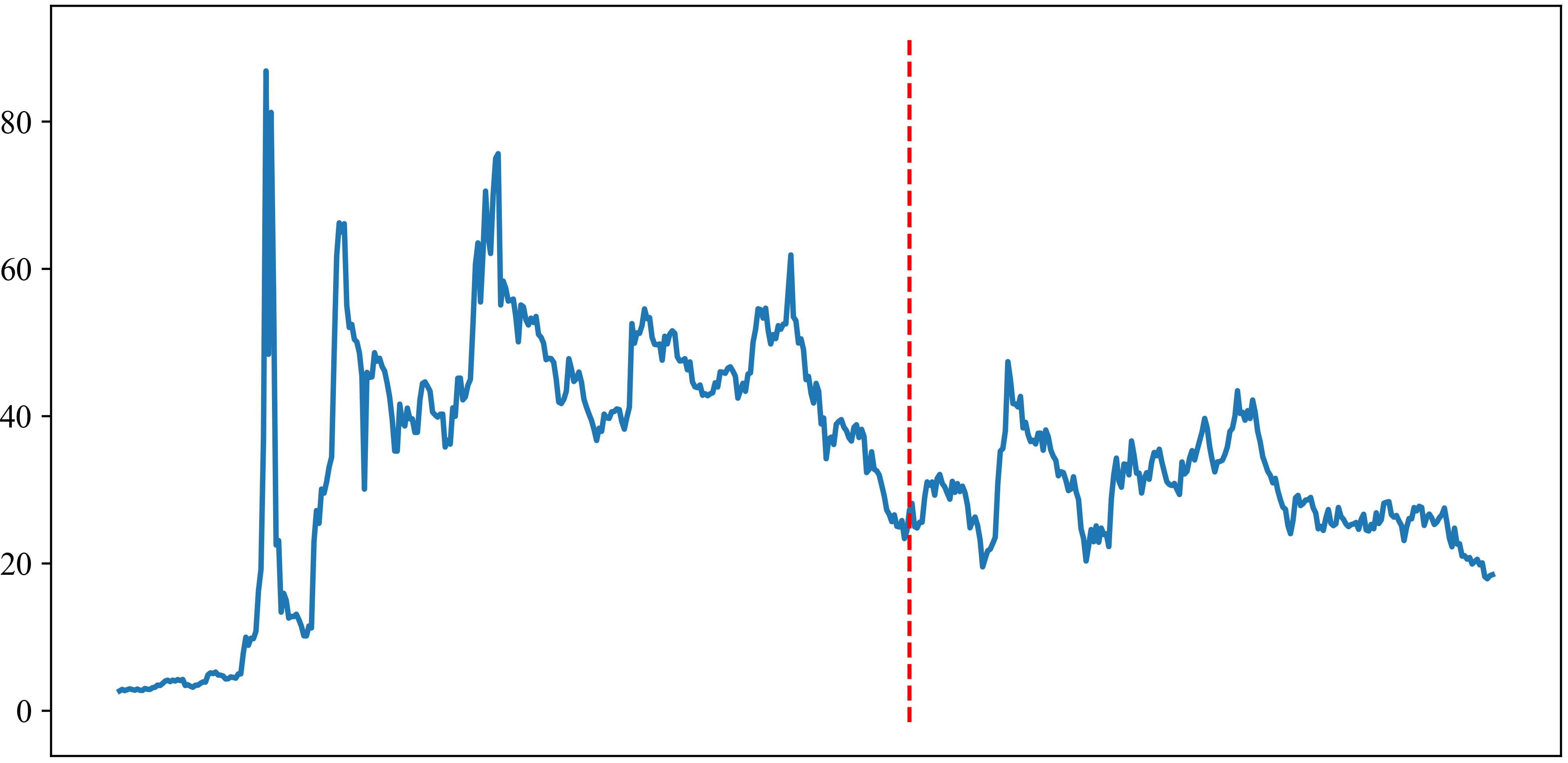}
		\caption{GME}
		\label{fig:gem}
	\end{subfigure}
	\caption{Trend plots of three stock datasets, where the red dashed lines in the figures indicate the time points dividing the training set and the test set.}
	\label{fig:financial}
\end{figure}

\begin{table}[!htbp]
	\centering
	\caption{Volatility and risk indicators for three stocks.}
	\begin{tabular}{lccc}
		\toprule
		Stock & Daily Return Std. Dev. & Annualized Volatility & Max Drawdown \\
		\midrule
		GOOG &  0.0203 & 0.3216 & 0.4460 \\
		RCL  & 0.0522 & 0.8284 & 0.8347 \\
		GME  & 0.1248 & 1.9807 & 0.8832 \\
		\bottomrule
	\end{tabular}
	\label{tab:stock_volatility}
\end{table}

\newpage

\begin{table*}[!htbp]
	\centering
     {\setlength{\tabcolsep}{1pt}
	\caption{PICP for 95\% prediction intervals of the financial datasets.}
    \label{tab:picp_finance}
	\begin{tabular}{l|cccccc}
    \hline
    Dataset & $\beta$-NLL & Ensemble-NN & NGBoost & SampleNet & MDN & NE-GMM \\
    \hline
    GOOG & 0.926$\pm$0.018 & 0.941$\pm$0.027 & \textbf{0.959$\pm$0.016} & 0.471$\pm$0.061 & \textbf{0.942$\pm$0.017} & \textbf{0.946$\pm$0.011} \\
    RCL & 0.711$\pm$0.171 & 0.537$\pm$0.078 & \textbf{0.936$\pm$0.071} & 0.678$\pm$0.161 & 0.783$\pm$0.268 & 0.895$\pm$0.299 \\
    GME & 0.604$\pm$0.046 & 0.919$\pm$0.025 & 0.985$\pm$0.004 & 0.470$\pm$0.077 & 0.935$\pm$0.051 & \textbf{0.953$\pm$0.055} \\     
    \hline
	\end{tabular}}
\end{table*}

\begin{table*}[!htbp]
	\centering
    {\setlength{\tabcolsep}{2pt}
	\caption{MPIW for 95\% prediction intervals of the financial datasets.}
    \label{tab:mpiw_finance}
	\begin{tabular}{l|cccccc}
    \hline
    Dataset & $\beta$-NLL & Ensemble-NN & NGBoost & SampleNet & MDN & NE-GMM \\
    \hline
    GOOG & 0.273$\pm$0.005 & 0.315$\pm$0.037 & 0.717$\pm$0.011 & 0.329$\pm$0.057 & 0.324$\pm$0.030 & 0.296$\pm$0.014 \\
    RCL & 0.540$\pm$0.056 & 0.743$\pm$0.147 & 1.477$\pm$0.032 & 0.676$\pm$0.291 & 0.896$\pm$0.245 & 0.905$\pm$0.391 \\
    GME & 0.468$\pm$0.042 & 2.818$\pm$1.592 & 1.045$\pm$0.059 & 0.815$\pm$0.146 & 1.035$\pm$0.448 & 1.587$\pm$0.665 \\
    \hline
	\end{tabular}}
\end{table*}

\begin{table*}[!htbp]
	\centering
	\caption{Average training time (in seconds) of the financial datasets.}
	\begin{tabular}{l|cccccc}
		\hline
		Dataset & $\beta$-NLL  & Ensemble-NN & NGBoost & SampleNet & MDN & NE-GMM \\
		\hline
		GOOG   & 99     & 321   & 5.724    & 241& 124  & 127     \\ 
		RCL   & 117  & 213  & 3.722    & 144& 67   & 94      \\ 
		GME  & 97   & 988  & 2.243    & 134& 86   & 88      \\ \hline
	\end{tabular}
	\label{tab:cost_finance}
\end{table*}

\vskip 0.2in
\bibliography{ref}

\end{document}